\pdfoutput=1
\documentclass{article}

\usepackage[main, final]{neurips_2026}

\usepackage[T1]{fontenc}    
\usepackage{hyperref}       
\usepackage{url}            
\usepackage{booktabs}       
\usepackage{amsfonts}       
\usepackage{nicefrac}       
\usepackage{microtype}      
\usepackage{xcolor}         

\usepackage{algorithm}
\usepackage{algorithmic}
\usepackage{float}
\usepackage{placeins}

\usepackage{xpatch}
\usepackage{array}
\usepackage{amsmath}
\usepackage{amssymb}
\usepackage{amsthm}
\usepackage{graphicx}
\usepackage{physics}
\usepackage{multirow}
\usepackage{subfigure}
\usepackage{mathtools}

\usepackage{tikz}
\usetikzlibrary{positioning, decorations.pathreplacing, calligraphy}

\newtheorem{definition}{Definition}
\newtheorem{lemma}{Lemma}
\newtheorem*{lemma*}{Lemma}
\newtheorem{proposition}{Proposition}
\newtheorem{theorem}{Theorem}
\newtheorem*{theorem*}{Theorem}
\newtheorem{corollary}{Corollary}
\newtheorem{assumption}{Assumption}
\newtheorem*{assumption*}{Assumption}
\theoremstyle{remark}
\newtheorem{remark}{Remark}
\theoremstyle{plain}

\DeclareFontShape{T1}{ptm}{m}{scit}{<->ssub * ptm/m/sc}{}
\DeclareFontShape{T1}{ptm}{b}{scit}{<->ssub * ptm/b/sc}{}
\newcolumntype{C}[1]{>{\centering\arraybackslash}m{#1}}
\makeatletter
\AtBeginDocument{\xpatchcmd{\@thm}{\thm@headpunct{.}}{\thm@headpunct{}}{}{}}
\makeatother

\title{Tree-Guided Identify-Then-Exploit: A Unified Framework of Best Arm Identification and Regret Minimization for Dueling Bandits}

\author{
  Pu Wang \\
  State Key Lab of CAD\&CG \\
  Zhejiang University \\
  \texttt{puwang0508@gmail.com}\\
  \And
  Yao-Xiang Ding \\
  State Key Lab of CAD\&CG \\
  Zhejiang University \\
  \texttt{dingyx.gm@gmail.com} \\
}

\begin{document}

\maketitle

\begin{abstract}
We study $N$-armed stochastic dueling bandits under the Condorcet-winner assumption with a positive global gap, where three widely adopted objectives are considered: best-arm identification (BAI), weak regret, and strong regret. We propose {\it Tree-Guided Identify-Then-Exploit} (\textsc{TG-ITE}), the first unified framework to tackle all these objectives to our knowledge. Without requiring transitivity assumptions, we propose a shared tree-guided identification approach to find a high-confidence incumbent within $O(N)$ comparisons. We further propose varied exploitation strategies to utilize this warm-start stage to optimize the specific objectives at hand. For a fixed global gap $\Delta$, this methodology enables our approach to (1) build the first winner-stays-style algorithm to achieve $O(N)$ weak regret; (2) realize the joint optimization of BAI and weak regret with $O(N)$ guarantees for both, eliminating the sub-optimal $O(\log N)$ factors of the existing approach; (3) attain the same $O(N\log T)$ order as specialized strong-regret approaches; and (4) achieve $O(N)$ sample complexity in BAI, matching a worst-case lower bound up to an additional $\log(1/\Delta)$ term. If a lower bound on the minimal winner gap is known, this term can be removed.
Our results suggest that a single identification primitive can serve BAI and both regret objectives with linear dependence on $N$.
\end{abstract}

\section{Introduction}
\label{sec:introduction}

Dueling bandits~\citep{yue2009interactively,bengs2021preference} model sequential learning from noisy relative feedback: instead of observing a numerical reward, the learner compares two arms and observes a binary preference outcome. This feedback model is natural in ranking~\citep{zoghi2015mergerucb}, recommendation~\citep{ma2025dual}, and human-preference learning~\citep{christiano2017deep,ouyang2022training,rafailov2023direct}, where relative judgments are easier to obtain than calibrated absolute scores. In this paper, we study $N$-armed dueling bandits under the Condorcet winner assumption. In this setting, there exists a single best arm---the Condorcet winner \(a^\star\)---that beats every other arm, while preferences among the remaining arms may be cyclic. We consider three common objectives in this setting: (1) \emph{best arm identification} (BAI), which requires returning \(a^\star\) at a data-dependent stopping time, as well as two regret minimization (RM) objectives, i.e. (2) \emph{online weak regret}, where comparisons involving \(a^\star\) are free; (3) \emph{online strong regret}, where only self-comparison \((a^\star,a^\star)\) is free.\footnote{We use the binary conventions of \citet{chen2017dueling}, formalized in Section~\ref{sec:pre}. They differ from terminal simple regret and from gap-weighted cumulative regret (Remark~\ref{rem:regret-conventions}).} Our target is to explore the following question: {\it Can we build a unified algorithmic framework that leads to optimally performed algorithms in all three objectives?}

Intuitively, these objectives agree on what must eventually be learned: the Condorcet winner \(a^\star\). But they disagree on how the learned winner should be exploited. This observation suggests a natural \emph{identify-then-exploit} design template: first, identify the Condorcet winner with high confidence, then exploit it according to the objective at hand. However, the Explore-Then-Commit (ETC) strategy for stochastic multi-armed bandit (MAB), which is built exactly on this clean idea, has been shown to be necessarily suboptimal on cumulative regret in the seminal work of~\citet{garivier2016etc}. This result is not incidental: The pioneering work of \citet{audibert2010best} has already shown that in MAB, BAI requires highly exploratory strategies. These results reveal a fundamental principle in bandits: BAI and RM involve different trade-offs in exploration and exploitation \citep{bubeck2009pure}. Generally, if one algorithm is optimal on one side, then it may not be optimal on another.

On the other hand, allowing to choose two arms at one time seems to make dueling bandits more benign between BAI and RM. This benefit allows an algorithm to utilize one arm for exploration and another for exploitation. This property especially fits for weak regret optimization: One could keep a current incumbent in one coordinate (the weak regret can be eliminated if the choice is correct) while testing challengers in another (to fix the possible wrong choice of incumbent). This is exactly the methodology behind the \emph{winner-stays} (\textsc{WS}) strategy for weak regret~\citep{chen2017dueling}. Furthermore, this naturally inspires us to utilize identify-then-exploit strategy to jointly optimize BAI and weak regret: Introducing a BAI warm start phase to identify the high-confidence incumbent, and then conducting \textsc{WS}-style exploitation for controlling weak regret. Recently, \citet{ma2025dual} proposes a \textsc{WS}-style algorithm to achieve this target. Denote by $N$ the number of arms. Under the assumptions that Condorcet winner exists and the arms have a total order, their approach achieves $O(N\log N)$ weak regret, which matches the \textsc{WS-W} algorithm from~\citep{chen2017dueling}, and $O(N\log N)$ BAI sample complexity, $O(\log N)$ factor worse than the existing optimal BAI algorithm. Though promising, important limitations still remain: (1) The $O(\log N)$ gap in BAI remains unclosed; (2) In comparison to the current best $O(N)$ weak regret algorithm from~\citep{saad2024weak}, which builds upon a different paradigm of online mirror descent with Tsallis regularizer, the result is still sub-optimal and depends on additional total order assumption; (3) As shown in \citep{chen2017dueling}, \textsc{WS} can be tuned to optimize strong regret, which is not considered in~\citet{ma2025dual}.

\begin{table}[t]
\centering
\small
\renewcommand{\arraystretch}{1.3}
\caption{Gap-oblivious guarantees under Assumption~\ref{assump:condorcet-gap}. The joint row bounds both the identification time and $R_\infty^{\rm weak}$. Section~\ref{sec:known-gap-budgeted-extension} treats the case of a known lower bound on the winner gap.}
\begin{tabular}{C{0.17\linewidth} C{0.32\linewidth} C{0.42\linewidth}}
\toprule
Objective & Policy & Guarantee (in expectation) \\
\midrule
Exact $\delta$-BAI & \textsc{TreeAscent}, then stop & $O\!\left(N\Delta^{-2}[\log(1/\delta)+\log(1/\Delta)]\right)$ \\
Weak regret & Constant-confidence warm start, then \textsc{ScreenAndReplace} & $O\!\left(N\Delta^{-2}\log(1/\Delta)\right)$ for $R_\infty^{\rm weak}$ \\
Joint BAI and weak regret & Report the $\delta$-confidence output, then continue screening & $O\!\left(N\Delta^{-2}[\log(1/\delta)+\log(1/\Delta)]\right)$ \\
Strong regret, known $H$ & Identify, then self-play & $O\!\left(N\Delta^{-2}[\log H+\log(1/\Delta)]\right)$ \\
Strong regret, anytime & Repeated fresh identifications and self-play blocks & $O\!\left(N\Delta^{-2}[\log T+\log(1/\Delta)\log\log T]\right)$ \\
\bottomrule
\end{tabular}
\label{tab:policies-summary}
\end{table}

In this work, we propose \emph{Tree-Guided Identify-Then-Exploit} (\textsc{TG-ITE}), the first unified framework for BAI, weak regret, and strong regret in dueling bandits to our knowledge. The main results are illustrated in Table~\ref{tab:policies-summary}. For each objective, we introduce a shared identification subroutine \textsc{TreeAscent}, followed by objective-specific exploitation strategies. \textsc{TreeAscent} serves as the common warm start and identification routine of the following exploitation stages. The key challenge lies in identifying an incumbent with high confidence within $O(N)$ comparisons without requiring assumptions beyond the existence of a Condorcet winner and a positive global gap. We address this by introducing the technique of \emph{tree-based tournament decomposition}, as detailed in Section~\ref{subsec:fixed-tree-path-decomposition}. For BAI, gap-oblivious \textsc{TreeAscent} matches the $O(N)$ bound of~\citet{falahatgar2017maximum} up to an additional $\log(1/\Delta)$ term, without transitivity assumptions. Given a lower bound on the winner gap, a budgeted variant (Section~\ref{sec:known-gap-budgeted-extension}) removes this term and matches a worst-case lower bound (Proposition~\ref{prop:bai-homogeneous-lower-bound}). For strong regret, we integrate \textsc{TreeAscent} with the self-play exploitation strategy, achieving $O(N\log T)$ strong regret for fixed gaps, which matches the leading order of the \textsc{WS-S} algorithm in~\citet{chen2017dueling}, as well as other classical algorithms specialized for strong regret minimization~\citep{yue2011beat,yue2012k,zoghi2014relative}. For weak regret, we integrate \textsc{TreeAscent} with a novel \textsc{WS}-style exploitation strategy \textsc{ScreenAndReplace}. The resulting \textsc{TG-ITE-Weak} algorithm improves the \textsc{WS-W} guarantee from $O(N^2)$, under the Condorcet assumption only, to $O(N)$, matching the linear dependence on $N$ of the current best result in~\citet{saad2024weak}. Notably, \textsc{TG-ITE-Weak} can be used to simultaneously conduct BAI with $O(N)$ sample complexity and achieve $O(N)$ weak regret, closing both gaps of $O(\log N)$ in~\citep{ma2025dual}. Overall, our unified framework suggests that, in dueling bandits, one identification primitive can serve both identification and regret minimization with linear dependence on $N$.\footnote{We delay more discussions of the comparison between our results and existing works to the remarks in Section~\ref{sec:identification-to-exploitation}. Due to space limitation, we leave full related work discussions in Appendix~\ref{sec:rw} and all formal proofs in the appendix.}

\section{Problem Formulation}
\label{sec:pre}

We study stochastic dueling bandits with a finite arm set $[N]\coloneqq\{1,\ldots,N\}$, $N\ge 2$. At each round $t$, the learner selects an ordered pair $(i_t,j_t)\in[N]^2$ based on the past and observes an outcome $Y_t\sim\mathrm{Bernoulli}(p_{i_t,j_t})$, where $Y_t=1$ indicates that $i_t$ wins the duel. Given the selected pair, $Y_t$ is independent of the past. The fixed, unknown preference matrix $P=(p_{i,j})_{i,j\in[N]}$ satisfies $p_{i,j}\in(0,1)$ and $p_{i,j}=1-p_{j,i}$. Self-comparisons are allowed and uninformative, since $p_{i,i}=1/2$.

The Condorcet winner, if it exists, is the unique arm $a^\star$ with $p_{a^\star,j}>1/2$ for all $j\ne a^\star$. Define the winner gaps $\Delta_i^\star\coloneqq p_{a^\star,i}-1/2$ and the pairwise gaps $\Delta_{i,j}\coloneqq|p_{i,j}-1/2|$. We call
\[
 \Delta^\star\coloneqq\min_{i\ne a^\star}\Delta_i^\star
 \qquad\text{and}\qquad
 \Delta\coloneqq\min_{i\ne j}\Delta_{i,j}
\]
the \emph{minimal winner gap} and the \emph{global gap}. Then $\Delta^\star>0$ and $\Delta\le\Delta^\star$, but $\Delta=0$ if nonwinners tie.

\begin{assumption}[Condorcet winner with positive global gap]
\label{assump:condorcet-gap}
A Condorcet winner exists, and $\Delta>0$.
\end{assumption}
Assumption~\ref{assump:condorcet-gap} allows cycles among the nonwinner arms. Section~\ref{sec:known-gap-budgeted-extension} removes the condition $\Delta>0$ when a lower bound $0<g\le\Delta^\star$ on the minimal winner gap is known.

\paragraph{Best-arm identification ($\delta$-BAI).}
In the fixed-confidence BAI problem, the learner receives a confidence parameter $\delta\in(0,1)$, stops at a data-dependent time $\tau$, and outputs an arm $\widehat a\in[N]$. The algorithm is \emph{$\delta$-correct} if $\mathbb P(\widehat a=a^\star)\ge 1-\delta$, and its performance is measured by the expected stopping time $\mathbb E[\tau]$. A BAI algorithm terminates and produces no further comparisons after stopping. In some settings, the lower bound $g$ is available as side information, which removes the $\log(1/\Delta)$ term from the sample complexity (Section~\ref{sec:known-gap-budgeted-extension}).

\paragraph{Regret minimization.}
At each round $t$, the learner incurs the (binary) \emph{weak regret} or \emph{strong regret}
\[
    r_t^{\mathrm{weak}}
    \coloneqq
    \mathbf 1\{i_t\neq a^\star,\ j_t\neq a^\star\},
    \qquad
    r_t^{\mathrm{str}}
    \coloneqq
    1-\mathbf 1\{i_t=a^\star,\ j_t=a^\star\},
    \qquad
    R_T^{x}
    \coloneqq
    \sum_{t=1}^T r_t^{x}
\]
for $x\in\{\mathrm{weak},\mathrm{str}\}$. Thus a round incurs zero weak regret whenever at least one of the two compared arms is the Condorcet winner, but zero strong regret only at $(a^\star,a^\star)$. Weak-regret guarantees are stated for the horizonless total $R_\infty^{\mathrm{weak}}\coloneqq\lim_{T\to\infty}R_T^{\mathrm{weak}}$ by default. For strong regret, we treat both \emph{known-horizon} algorithms, which take $T$ as input, and \emph{horizonless} algorithms, which are evaluated at every finite $T$ without knowing it in advance.

\begin{remark}[Binary, gap-weighted, and terminal criteria]
\label{rem:regret-conventions}
The binary losses above follow \citet{chen2017dueling}. A common alternative uses the \emph{gap-weighted losses} $r_t^{{\rm weak},{\rm gw}}=\min\{\Delta_{i_t}^\star,\Delta_{j_t}^\star\}$ and $r_t^{{\rm str},{\rm gw}}=(\Delta_{i_t}^\star+\Delta_{j_t}^\star)/2$, which have the same zero-loss actions and are related to the binary losses in Appendix~\ref{app:weak-lower-scope}. A separate criterion, terminal \emph{simple regret}, evaluates only the final recommendation, for example through $\Delta_{\widehat a}^\star$ or $\mathbf 1\{\widehat a\ne a^\star\}$. The losses above instead account for the entire interaction.
\end{remark}

\paragraph{Objective-specific instantiations.}
Our goal is not a single policy that is simultaneously optimal for all three criteria. Instead, the shared tree-guided identification primitive \textsc{TreeAscent} is instantiated with objective-specific confidence schedules and exploitation tails: termination after a high-confidence identification (BAI), a champion-centric online tail (weak regret), and interleaved identifications with self-play (strong regret).

\begin{algorithm}[h]
\caption{\textsc{DuelTest}$(u,v,\delta)$: Pairwise testing over two arms}
\label{alg:dueltest}

\begin{algorithmic}[1]
\REQUIRE Arms $u\neq v$ and confidence level $\delta\in(0,1)$.
\STATE $n\leftarrow 0$, $S\leftarrow 0$.
\REPEAT
    \STATE $n\leftarrow n+1$.
    \STATE Observe $Y_n\sim\mathrm{Bernoulli}(p_{u,v})$, where $Y_n=1$ means that $u$ wins.
    \STATE $S\leftarrow S+Y_n$ and $\widehat\mu_n\leftarrow S/n$.
\UNTIL{$|\widehat\mu_n-1/2|>c(n,\delta)$, where $c(n,\delta) = \sqrt{\frac{\log\bigl(2n(n+1)/\delta\bigr)}{2n}}$.}
\IF{$\widehat\mu_n>1/2$}
    \STATE \textbf{return} $u$.
\ELSE
    \STATE \textbf{return} $v$.
\ENDIF
\end{algorithmic}
\end{algorithm}

\begin{algorithm}[h]
\caption{\textsc{BKT}$(S,\delta)$: Balanced knockout tournament on a subset of arms}
\label{alg:bkt}
\begin{algorithmic}[1]
\REQUIRE A nonempty arm set $S\subseteq[N]$ and confidence level $\delta\in(0,1)$.
\STATE $X\leftarrow S$, $r\leftarrow1$.
\WHILE{$|X|>1$}
    \STATE Set $\delta_r\leftarrow 6\delta/(\pi^2 r^2)$.
    \STATE Pair the arms in $X$ according to a fixed deterministic order; one arm receives a bye if $|X|$ is odd.
    \STATE For each pair $(u,v)$, keep the winner returned by \textsc{DuelTest}$(u,v,\delta_r)$.
    \STATE Let $X$ be the set of winners together with the possible bye arm; set $r\leftarrow r+1$.
\ENDWHILE
\STATE \textbf{return} the unique arm in $X$.
\end{algorithmic}
\end{algorithm}

\section{Tree-Guided Identification}
\label{sec:tree-guided-identification}

This section develops the identification phase used by all later policies. The primitive \textsc{TreeAscent} takes an arbitrary starting arm \(b\) and returns a high-confidence candidate for the Condorcet winner using \(O(N)\) expected comparisons for a fixed global gap and confidence level, uniformly over \(b\). Section~\ref{sec:identification-to-exploitation} then attaches objective-specific exploitation policies for the three objectives.

The design of \textsc{TreeAscent} starts from a graph-level picture. Under Assumption~\ref{assump:condorcet-gap}, the arms and their pairwise preferences form a (possibly cyclic) directed graph in which the Condorcet winner \(a^\star\) uniquely dominates every other arm. We overlay this graph with a deterministic balanced binary tree \(\mathcal T\) on the \(N\) arms. \(\mathcal T\) is not learned from data and serves only as a geometric scaffold that prescribes which comparisons to run: it groups the arms into sibling blocks of geometrically growing size around \(b\), so that a size-proportional confidence allocation keeps the total identification cost linear in \(N\). The knockout comparisons induced by \(\mathcal T\) then carve out a directed beating tree, which contains no cycles even when the preference graph does, and the surviving candidate is the arm at the top of this tree. As the dominating arm, \(a^\star\) reaches the top with high probability: it either starts as \(b\) or is selected by the tournament on the block containing it, and every later merge preserves it. This capture-and-preservation argument needs only the comparisons involving \(a^\star\) to be correct, so the nonwinners never need to be ordered correctly and no transitivity assumption is required. For readability the main text assumes \(N=2^L\), and the general case uses formal byes (Appendix~\ref{app:tlr-proofs}).

\subsection{Statistical Primitives}
\label{subsec:statistical-primitives}

We first introduce two algorithmic subroutines. The first (Algorithm~\ref{alg:dueltest}) is a sequential test for a single pair. It repeatedly compares \(u\) and \(v\) until the empirical mean separates from \(1/2\) by an anytime confidence radius, and then returns the empirically better arm. The second (Algorithm~\ref{alg:bkt}) is a balanced knockout tournament that returns a representative of an arbitrary subset. Lemma~\ref{lem:app-bkt-subset} shows that its comparison cost is linear in the subset size, and that if the subset contains \(a^\star\), then the tournament returns that arm with high probability.

\subsection{Tree-Based Tournament Decomposition}
\label{subsec:fixed-tree-path-decomposition}
We now make the geometric scaffold concrete. The key mechanism is to decompose global identification of the Condorcet winner into a series of local tournaments with geometrically growing sizes, organized along the ancestor path of the starting arm \(b\). Figure~\ref{fig:sibling-blocks} illustrates the construction.

Fix a deterministic balanced binary tree \(\mathcal T\) on the \(N=2^L\) arms, and write \(S(v)\) for the set of arms in the subtree below an internal node \(v\). Now fix a starting arm \(b\in[N]\) and walk from the leaf \(\{b\}\) up to the root. At each level \(h\in\{0,1,\ldots,L\}\), this walk reaches an ancestor whose subtree covers \(2^h\) arms; we call this set the \emph{ancestor block}
\[
    U_h(b),\qquad U_0(b)=\{b\},\qquad U_L(b)=[N],\qquad |U_h(b)|=2^h.
\]
By the binary structure, \(U_h(b)\) splits into two halves of equal size: the previous ancestor block \(U_{h-1}(b)\), and the arms freshly attached at level \(h\), which we call the \emph{sibling block}
\[
    Q_h(b)\;:=\;U_h(b)\setminus U_{h-1}(b),\qquad |Q_h(b)|=2^{h-1}.
\]
Thus \(Q_1(b)\) is a single arm and \(Q_L(b)\) already contains half of all arms. Iterating the split yields the disjoint partition
\[
    [N]=\{b\}\sqcup Q_1(b)\sqcup\cdots\sqcup Q_L(b),\qquad m_h:=|Q_h(b)|=2^{h-1},\qquad \sum_{h=1}^L m_h=N-1.
\]
The Condorcet winner \(a^\star\) therefore either coincides with \(b\) or lies in exactly one sibling block. \textsc{TreeAscent} will identify which by walking the ancestor path and running a knockout tournament on each block.

\begin{figure}[t]
\centering
\begin{tikzpicture}[
    scale=.7,
    every node/.style={font=\small},
    pnode/.style={draw, rounded corners=1.5pt, fill=black!5,
                  inner sep=3pt, minimum height=0.5cm, font=\small},
    qnode/.style={draw, rounded corners=1.5pt, fill=blue!6, densely dashed,
                  inner sep=3pt, minimum height=0.5cm, font=\small,
                  draw=blue!40!black},
    arr/.style={->, >=stealth, semithick, black!60},
    qarr/.style={->, >=stealth, semithick, blue!40!black},
    lbl/.style={font=\footnotesize, black!55},
    qlbl/.style={font=\footnotesize, blue!40!black},
]

\def\vs{1.35}  
\def\hs{1.8}   

\node[pnode] (root) at (0, 0) {$\mathrm{root}=U_L(b)$};

\node[pnode] (UL1) at (-\hs, -\vs) {$U_{L\!-\!1}(b)$};
\node[qnode] (QL)  at ( \hs, -\vs) {$Q_L(b)$};

\node[pnode] (UL2) at (-2*\hs+\hs/2, -2*\vs) {$U_{L\!-\!2}(b)$};
\node[qnode] (QL1) at (-\hs+\hs/2+0.15, -2*\vs) {$Q_{L\!-\!1}(b)$};

\node[font=\large] (dots) at (-2*\hs+\hs/2, -2*\vs - 0.85) {$\vdots$};

\node[pnode] (U1)  at (-2*\hs+\hs/2, -2*\vs - 1.7) {$U_1(b)$};

\node[pnode] (leaf) at (-2*\hs, -2*\vs - 1.7 - \vs) {$\{b\}$};
\node[qnode] (Q1)   at (-2*\hs+\hs, -2*\vs - 1.7 - \vs) {$Q_1(b)$};

\draw[arr] (root) -- (UL1);
\draw[arr] (UL1)  -- (UL2);
\draw[arr] (U1)   -- (leaf);

\draw[qarr] (root) -- (QL);
\draw[qarr] (UL1)  -- (QL1);
\draw[qarr] (U1)   -- (Q1);

\node[qlbl, right=3pt of QL]  (szL)  {$2^{L-1}$ arms};
\node[qlbl, right=3pt of QL1] (szL1) {$2^{L-2}$ arms};
\node[qlbl, right=3pt of Q1]  (sz1)  {$1$ arm};

\node[lbl, left=6pt of root] {level $L$};
\node[lbl, left=6pt of UL1]  {level $L\!-\!1$};
\node[lbl, left=6pt of UL2]  {level $L\!-\!2$};
\node[lbl] at (-2*\hs+\hs/2 - 1.45, -2*\vs - 1.7) {level $1$};
\node[lbl, left=6pt of leaf] {level $0$};

\coordinate (btop) at ([xshift=18pt]szL.east);
\coordinate (bbot) at ([xshift=12pt]sz1.east);
\draw[decorate, decoration={brace, amplitude=5pt},
      blue!40!black, semithick]
    (btop |- QL.north) -- (bbot |- Q1.south)
    node[midway, right=10pt, font=\footnotesize, blue!40!black]
    {$\displaystyle\sum_h\!|Q_h|=N\!-\!1$};

\node[pnode, minimum width=0.55cm, minimum height=0.28cm,
      font=\scriptsize] (leg1box) at (2.9, 0.22) {};
\node[font=\scriptsize, right=3pt of leg1box] {ancestor path};

\node[qnode, minimum width=0.55cm, minimum height=0.28cm,
      font=\scriptsize] (leg2box) at (2.9, -0.23) {};
\node[font=\scriptsize, right=3pt of leg2box] {sibling blocks};

\end{tikzpicture}
\caption{The ancestor path of a starting arm~$b$ decomposes the arms into
sibling blocks $Q_1(b),\ldots,Q_L(b)$.
\textsc{TreeAscent} walks up the ancestor path, runs a knockout tournament
on each sibling block to obtain a representative, and merges that representative
with the current champion. If a block contains $a^\star$, the representative is
$a^\star$ with high probability. Block sizes grow geometrically ($|Q_h|=2^{h-1}$),
enabling a size-proportional confidence allocation that keeps the total
identification cost linear in~$N$.}
\label{fig:sibling-blocks}
\end{figure}

\subsection{Global Winner Identification by \textsc{TreeAscent}}
\label{subsec:tree-ascent-pass}

Based on the tree-based tournament decomposition, we establish the tree-guided identification algorithm as illustrated in Algorithm~\ref{alg:plr-pass}. At each level, the pass runs a knockout tournament on the newly attached sibling block and merges the returned representative with the current champion. The proof only tracks the first level at which \(a^\star\) enters the path, and then shows that correct merge tests preserve it.

\begin{algorithm}[t]
\caption{\textsc{TreeAscent}\((b,\mathcal T,\epsilon)\): Tree-guided winner identification}
\label{alg:plr-pass}
\begin{algorithmic}[1]
\REQUIRE Starting arm \(b\in[N]\); balanced tree \(\mathcal T\) on \(N=2^L\) leaves; confidence \(\epsilon\in(0,1)\).
\ENSURE An identified candidate \(\widetilde b\).
\STATE Compute \(v_0(b)\subset\cdots\subset v_L(b)=\mathrm{root}\) and sibling blocks \(Q_1(b),\ldots,Q_L(b)\).
\STATE Set \(c_0\leftarrow b\).
\FOR{\(h=1,2,\ldots,L\)}
    \STATE \(m_h\leftarrow |Q_h(b)|\).
    \STATE \(\displaystyle \beta_h\leftarrow \frac{\epsilon}{2}\frac{m_h}{N-1}\), \qquad \(\displaystyle \eta_h\leftarrow \frac{\epsilon}{2L}\).
    \STATE \(g_h\leftarrow \textsc{BKT}(Q_h(b),\beta_h)\).
    \STATE \(c_h\leftarrow \textsc{DuelTest}(c_{h-1},g_h,\eta_h)\).
\ENDFOR
\STATE \textbf{return} \(\widetilde b\leftarrow c_L\).
\end{algorithmic}
\end{algorithm}

The subtree budgets \(\beta_h\) are proportional to the block sizes \(m_h\): this proportionality is what keeps \(\sum_h m_h\log(1/\beta_h)\) linear in \(N\), rather than \(N\log\log N\) as a uniform allocation would give. The merge budgets \(\eta_h\) are uniform and contribute at most \(\epsilon/2\) total failure probability.

\begin{theorem}[Tree-guided identification]
\label{thm:plr-repair-main}
Under Assumption~\ref{assump:condorcet-gap}, for every fixed balanced tree \(\mathcal T\), every starting arm \(b\in[N]\), and every \(\epsilon\in(0,1)\), the output of \textsc{TreeAscent}\((b,\mathcal T,\epsilon)\) satisfies
\[
    \mathbb P(\widetilde b=a^\star)\ge 1-\epsilon.
\]
Moreover, if \(\tau_{\mathrm{TA}}(b,\epsilon)\) denotes the total number of comparisons used by the pass, then
\[
    \mathbb E[\tau_{\mathrm{TA}}(b,\epsilon)]
    \le
    \frac{C N}{\Delta^2}
    \left(
        \log\frac1\epsilon+
        \log\frac1\Delta+1
    \right)
\]
for a universal constant \(C>0\). The same statement holds for arbitrary \(N\ge2\) under the formal-bye implementation in Appendix~\ref{app:tlr-proofs}.
\end{theorem}

We say that a subroutine call uses \emph{fresh samples} when every comparison outcome drawn during the call has not been used by the algorithm before. The pair choices inside the call may still be adaptive. Since Theorem~\ref{thm:plr-repair-main} holds for every starting arm, it also applies when $b$ is chosen based on earlier comparisons, provided the pass uses fresh samples.

\section{Objective-Specific Exploitation}
\label{sec:identification-to-exploitation}

\subsection{BAI: Identify Then Stop}
\label{subsec:bai-policy}

The BAI policy \textsc{TG-ITE-BAI} runs \(\textsc{TreeAscent}(b_0,\mathcal T,\delta)\), outputs the returned arm \(\widehat a\), and stops.

\begin{corollary}[Fixed-confidence $\delta$-BAI]
\label{cor:tlcr-bai}
Under Assumption~\ref{assump:condorcet-gap}, for every \(\delta\in(0,1)\), \textsc{TG-ITE-BAI} satisfies
\[
    \mathbb P(\widehat a=a^\star)\ge 1-\delta
\]
and its expected stopping time is bounded by
\[
    \mathbb E[\tau_{\mathrm{BAI}}]
    \le
    \frac{C N}{\Delta^2}
    \left(
        \log\frac1\delta+
        \log\frac1\Delta+1
    \right).
\]
\end{corollary}

\paragraph{Remark.} For fixed gap and confidence parameters, the result in Corollary~\ref{cor:tlcr-bai} guarantees $O(N)$ sample complexity without requiring commonly adopted assumptions such as strong stochastic transitivity and stochastic triangle inequality. Proposition~\ref{prop:bai-homogeneous-lower-bound} in Appendix~\ref{app:bai-lower-bound} gives a worst-case lower bound $\Omega(N\Delta^{-2}\log(1/\delta))$ on instances in which every pairwise gap equals $\Delta$, even for learners that know $\Delta$. The upper bound matches this rate up to an additional $\log(1/\Delta)$ term, which comes from the anytime confidence radius of \textsc{DuelTest}. The detailed discussions are provided in Section~\ref{rw:bai}.

\subsection{Weak Regret: Winner-Centric Exploitation}
\label{subsec:weak-policy}
The weak-regret policy uses \textsc{TreeAscent} for a constant-confidence warm start to identify an initial incumbent, then runs the iterative winner-centric exploitation subroutine \textsc{ScreenAndReplace} (Algorithm~\ref{alg:screen-and-replace}). In each epoch, the algorithm screens all challengers against the current incumbent and, if any appears to beat it, runs a replacement tournament on the incumbent together with those apparent beaters. The resulting weak-regret algorithm \textsc{TG-ITE-Weak} is illustrated in Algorithm~\ref{alg:tlcr-weak}. Each screening comparison includes the incumbent, so when the incumbent is $a^\star$, weak regret arises only if a challenger wrongly appears to beat it, which happens with probability at most $\rho$ per challenger. When the incumbent is not $a^\star$, the screen also compares it with $a^\star$, and the epoch returns $a^\star$ with probability at least $1-\rho-\kappa$. Because Algorithm~\ref{alg:tlcr-weak} shrinks $\rho_s$ and $\kappa_s$ geometrically, such errors quickly become rare and the total weak regret stays finite.

\begin{algorithm}[t]
\caption{\textsc{ScreenAndReplace}\((c,\rho,\kappa)\)}
\label{alg:screen-and-replace}
\begin{algorithmic}[1]
\REQUIRE Incumbent arm \(c\); screening confidence \(\rho\); replacement confidence \(\kappa\).
\STATE \(B\leftarrow\emptyset\).
\FOR{each \(j\in[N]\setminus\{c\}\) in a fixed order}
    \STATE \(w\leftarrow \textsc{DuelTest}(c,j,\rho)\) using fresh samples.
    \IF{\(w=j\)}
        \STATE \(B\leftarrow B\cup\{j\}\).
    \ENDIF
\ENDFOR
\IF{\(B=\emptyset\)}
    \STATE \textbf{return} \(c\).
\ELSE
    \STATE \textbf{return} \textsc{BKT}$(B\cup\{c\},\kappa)$ using fresh samples.
\ENDIF
\end{algorithmic}
\end{algorithm}

\begin{algorithm}[t]
\caption{\textsc{TG-ITE-Weak}\((b_0,\mathcal T)\)}
\label{alg:tlcr-weak}
\begin{algorithmic}[1]
\REQUIRE Starting arm \(b_0\in[N]\); fixed balanced tree \(\mathcal T\).
\STATE Run a constant-confidence warm start \(c_0\leftarrow \textsc{TreeAscent}(b_0,\mathcal T,1/4)\).
\FOR{\(s=1,2,3,\ldots\)}
    \STATE Set \(\rho_s\leftarrow 4^{-(s+2)}\) and \(\kappa_s\leftarrow 4^{-(s+2)}\).
    \STATE \(c_s\leftarrow \textsc{ScreenAndReplace}(c_{s-1},\rho_s,\kappa_s)\).
\ENDFOR
\end{algorithmic}
\end{algorithm}

\begin{theorem}[Online weak regret with linear dependence on \(N\)]
\label{thm:online-weak-linear-main} 
Under Assumption~\ref{assump:condorcet-gap}, the weak regret of Algorithm~\ref{alg:tlcr-weak} satisfies
\[
    \mathbb E[R_\infty^{\mathrm{weak}}]
    \le
    \frac{C N}{\Delta^2}
    \left(
        \log\frac1\Delta+1
    \right)
\]
for a universal constant \(C>0\). In particular, \(R_\infty^{\mathrm{weak}}<\infty\) almost surely.
\end{theorem}

Furthermore, by adjusting the warm start confidence in Algorithm~\ref{alg:tlcr-weak} to $\delta$, we can obtain an algorithm, \textsc{TG-ITE-Dual}, jointly optimizing BAI and weak regret. It reports $c_0$ as its BAI answer when the warm start ends and then continues screening.
\begin{corollary}[Joint BAI and weak regret]
\label{cor:dual}
Under Assumption~\ref{assump:condorcet-gap}, for every $\delta\in(0,1)$, \textsc{TG-ITE-Dual} satisfies $\Pr(c_0=a^\star)\ge1-\delta$, and its answer time $\tau_{\mathrm{Dual}}$ and weak regret satisfy
\[
\max\left\{
\mathbb E[\tau_{\mathrm{Dual}}],\;
\mathbb E[R_\infty^{\mathrm{weak}}]
\right\}
\le
\frac{CN}{\Delta^2}
\left(\log\frac1\delta+\log\frac1\Delta+1\right).
\]
\end{corollary}

Its identification stage is the same \textsc{TreeAscent} call as \textsc{TG-ITE-BAI}, so with the same inputs and feedback, both report the same arm at the same time. The $\log(1/\delta)$ term in the weak regret pays for this high-confidence warm start.

\paragraph{Remark.} Under the positive global gap assumption, our approach is the first winner-stays (\textsc{WS}) style algorithm that achieves $O(N)$ weak regret, improving the original \textsc{WS-W} algorithm~\citep{chen2017dueling}, whose guarantee is $O(N^2)$ under the Condorcet assumption and $O(N\log N)$ under a total order. Corollary~\ref{cor:dual} improves the sub-optimal $O(\log N)$ factors in \citet{ma2025dual}, yielding bounds that are linear in $N$ for both BAI and weak regret, although their sample-complexity guarantee is a high-probability bound, whereas ours bounds the expected answer time.

Unlike these binary-loss analyses, \citet{saad2024weak} bound gap-weighted weak regret. Converting the losses round by round (Appendix~\ref{app:weak-lower-scope}) shows that their algorithm attains binary weak regret $O(N/(\Delta^\star)^2)$ and that this order is unavoidable in the worst case. Theorem~\ref{thm:online-weak-linear-main} matches its linear dependence on $N$ but depends on $\Delta$ instead of $\Delta^\star$. Proposition~\ref{prop:global-gap-obstruction} shows that Algorithm~\ref{alg:tlcr-weak} actually incurs this dependence, because its pairwise tests can spend arbitrarily long on two nearly tied nonwinners even when the winner gap stays fixed.

\begin{proposition}[Global-gap obstruction for the weak-regret policy]
\label{prop:global-gap-obstruction}
Fix $d\in(0,1/4)$. For every sufficiently small $\gamma\in(0,d)$, there is a four-arm instance, fixed tree, and starting arm with $\Delta^\star=d$ and $\Delta=\gamma$ such that Algorithm~\ref{alg:tlcr-weak} satisfies
\[
    \mathbb E[R_\infty^{\rm weak}]
    \ge c_d\,\gamma^{-2}\log(1/\gamma)
\]
for some $c_d>0$ independent of $\gamma$. Consequently, the unchanged weak policy has no finite uniform bound depending only on $N$ and $\Delta^\star$.
\end{proposition}

\subsection{Strong Regret: Self-Play Exploitation}
\label{subsec:strong-policy}

Binary strong regret has the opposite geometry: any comparison of distinct arms incurs one unit of regret, even if one of them is \(a^\star\). The policy therefore identifies a reliable incumbent and self-plays.

\paragraph{Known horizon.}
If the horizon \(H\) is known, the policy (Algorithm~\ref{alg:tlcr-strong-known} in Appendix~\ref{app:strong-known-horizon}) runs \textsc{TreeAscent} at confidence \(\epsilon_H=1/(H+1)\) and then self-plays.

\begin{theorem}[Known-horizon strong regret]
\label{thm:strong-known-main}
Under Assumption~\ref{assump:condorcet-gap}, Algorithm~\ref{alg:tlcr-strong-known} satisfies, for every horizon \(H\ge1\),
\[
    \mathbb E[R_H^{\mathrm{str}}]
    \le
    \frac{C N}{\Delta^2}
    \left(
        \log(H+1)+
        \log\frac1\Delta+1
    \right)+1.
\]
\end{theorem}

\paragraph{Unknown horizon.}
When the horizon is unknown, the policy cannot choose a single confidence level \(1/(T+1)\). An anytime alternative is to interleave increasingly confident identified incumbents with increasingly long self-play blocks.

\begin{algorithm}[t]
\caption{\textsc{TG-ITE-Strong-Anytime}\((b_0,\mathcal T)\)}
\label{alg:tlcr-strong-anytime}
\begin{algorithmic}[1]
\REQUIRE Starting arm \(b_0\); fixed balanced tree \(\mathcal T\).
\STATE \(c_0\leftarrow b_0\).
\FOR{\(s=1,2,3,\ldots\)}
    \STATE Set \(\epsilon_s\leftarrow 2^{-2^s}\) and \(B_s\leftarrow 2^{2^s}\).
    \STATE \(c_s\leftarrow \textsc{TreeAscent}(c_{s-1},\mathcal T,\epsilon_s)\) using fresh samples.
    \STATE Play \((c_s,c_s)\) for the next \(B_s\) rounds.
\ENDFOR
\end{algorithmic}
\end{algorithm}

\begin{theorem}[Horizonless strong regret]
\label{thm:strong-anytime-main}
Under Assumption~\ref{assump:condorcet-gap}, Algorithm~\ref{alg:tlcr-strong-anytime} satisfies, for every \(T\ge 2\),
\[
    \mathbb E[R_T^{\mathrm{str}}]
    \le
    \frac{C N}{\Delta^2}
    \left(
        \log(T+1)
        +
        (\log\log(T+2)+1)
        \left(
            \log\frac1\Delta+1
        \right)
    \right)
    +C'(\log\log(T+2)+1)
\]
for universal constants \(C,C'>0\). For any fixed \(\Delta>0\), the leading term as \(T\to\infty\) is \(O(N\Delta^{-2}\log T)\).
\end{theorem}

\paragraph{Remark.} Existing strong regret specialized algorithms~\citep{yue2011beat,yue2012k,pmlr-v162-saha22a} usually achieve $O(N\log T/\Delta^\star)$ under gap-weighted loss. Converting the losses as in Section~\ref{subsec:weak-policy} turns this rate into binary strong regret $O(N\log T/(\Delta^\star)^2)$, and Proposition~\ref{prop:strong-binary-homogeneous-lower} in Appendix~\ref{app:strong-lower-scope} shows that this order is asymptotically unavoidable for strongly consistent policies on total orders in which every pairwise gap equals $\Delta^\star$. In comparison, for fixed gaps our results are optimal in $N, T$, but depend on $\Delta$ instead of $\Delta^\star$. Under gap-weighted loss, our general bounds retain this dependence on $\Delta$, because \textsc{TreeAscent} also tests nonwinner pairs at high confidence. Improving this dependence while keeping an explicit BAI guarantee remains open.

\section{Known-Gap Budgeted Extension}
\label{sec:known-gap-budgeted-extension}

This section assumes only that a Condorcet winner exists and that the learner knows a lower bound $0<g\le\Delta^\star$ on the minimal winner gap. As in Section~\ref{sec:tree-guided-identification}, whenever every test involving $a^\star$ returns $a^\star$, \textsc{BKT} returns $a^\star$ if it is a candidate, and \textsc{TreeAscent} and \textsc{ScreenAndReplace} return $a^\star$. In each such test, $a^\star$ beats its opponent by at least $\Delta^\star\ge g$. We can therefore replace \textsc{DuelTest} with \textsc{BudgetedDuelTest} (Algorithm~\ref{alg:budgeted-dueltest} in Appendix~\ref{app:known-gap-budgeted-proofs}), a variant of the comparison primitive of \citet{falahatgar2017maximum}. It caps each test at a budget fixed by $g$ and $\delta$, stops earlier once the empirical mean separates from $1/2$ by its confidence radius minus $g$, and returns the arm with more wins. A test involving $a^\star$ still returns $a^\star$ with the required confidence. A test between two nonwinners, even a tied one, can stop early or at the budget, and its answer does not affect correctness.

\begin{lemma}[Validity of the budgeted pairwise test]
\label{lem:budgeted-dueltest-main}
For every pair \(u\neq v\), Algorithm~\ref{alg:budgeted-dueltest} uses at most \(\lceil(2g^2)^{-1}\log(4/\delta)\rceil\) comparisons deterministically. If \(|p_{u,v}-1/2|\ge g\), then it returns the better arm with probability at least \(1-\delta\).
\end{lemma}

The tree, the confidence schedules, and the exploitation rules stay unchanged, so the pairwise guarantee of Lemma~\ref{lem:budgeted-dueltest-main} propagates to all objectives as follows.

\begin{theorem}[Known-gap budgeted \textsc{TG-ITE}]
\label{thm:known-gap-extension-main}
Assume that a Condorcet winner exists and that the learner knows $0<g\le\Delta^\star$, and run the policies of Section~\ref{sec:identification-to-exploitation} with \textsc{BudgetedDuelTest} in place of \textsc{DuelTest}. For universal constants $C,C'>0$ and every $\delta\in(0,1)$, the following bounds hold.
\[
\begin{gathered}
\text{Exact $\delta$-BAI: }\Pr(\widehat a=a^\star)\ge1-\delta,\quad
\tau_{\rm BAI}\le \frac{CN}{g^2}\left(\log\frac1\delta+1\right)
\text{ deterministically};\\
\text{Weak regret: }\mathbb E[R_\infty^{\rm weak}]\le \frac{CN}{g^2};\\
\text{Weak regret with a }\delta\text{-BAI warm start:}\quad
\mathbb E[R_\infty^{\rm weak}]
\le \frac{CN}{g^2}\left(\log\frac1\delta+1\right);\\
\text{Strong regret, known }H:\quad
\mathbb E[R_H^{\rm str}]
\le \frac{CN}{g^2}\left(\log(H+1)+1\right)+1;\\
\text{Strong regret, anytime: for every }T\ge2,\\
\mathbb E[R_T^{\rm str}]
\le \frac{CN}{g^2}\left(\log(T+1)+\log\log(T+2)+1\right)
+C'(\log\log(T+2)+1).
\end{gathered}
\]
\end{theorem}

Compared with Table~\ref{tab:policies-summary}, the global gap $\Delta$ is replaced by $g$, which can be as large as $\Delta^\star$, the $\log(1/\Delta)$ terms disappear, and the BAI stopping time is bounded deterministically. Since only the pairwise test has changed, all of the dependence on $\Delta$ in Table~\ref{tab:policies-summary} comes from gap-oblivious testing rather than from the tree. In the setting of Proposition~\ref{prop:bai-homogeneous-lower-bound}, with $g$ equal to the common gap, the BAI budget matches the lower bound up to constants.

\begin{figure}[h]
    \centering
    \includegraphics[width=\linewidth]{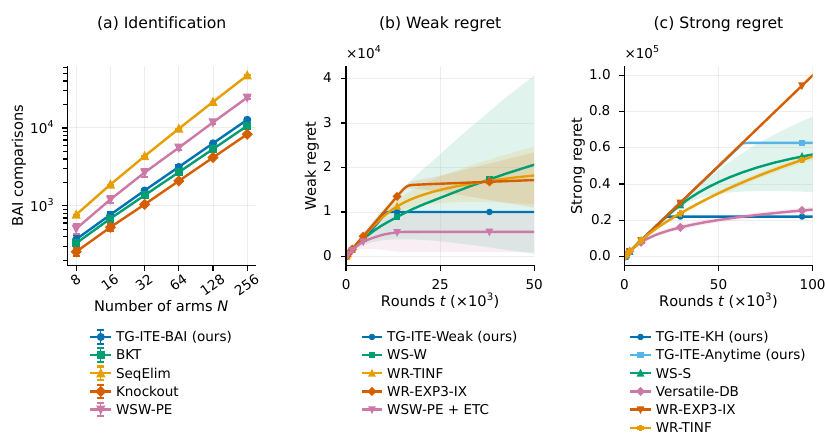}
    \caption{(a)~Cyclic BAI. (b)~Perturbed cyclic weak regret. (c)~Cyclic strong regret. Error bars and bands show one sample SD over 30 seeds.}
    \label{fig:experiments}
\end{figure}
\section{Experiments}
\label{sec:experiments}

On synthetic preference matrices, we evaluate \textsc{TG-ITE-BAI}, \textsc{TG-ITE-Weak}, and the known-horizon and anytime strong-regret policies, abbreviated \textsc{TG-ITE-KH} and \textsc{TG-ITE-Anytime}. The baselines are \textsc{BKT} (Algorithm~\ref{alg:bkt} on all arms), \textsc{SeqElim}~\citep{falahatgar2017few}, \textsc{Knockout}~\citep{falahatgar2017maximum}, \textsc{WS-W} and \textsc{WS-S}~\citep{chen2017dueling}, \textsc{WR-TINF} and \textsc{WR-EXP3-IX}~\citep{saad2024weak}, \textsc{WSW-PE}~\citep{ma2025dual}, and \textsc{Versatile-DB}~\citep{pmlr-v162-saha22a}. Since \textsc{WSW-PE} stops after outputting an arm, its online runs add an explore-then-commit continuation (\textsc{WSW-PE+ETC}). In the cyclic instances, the winner beats every other arm with gap $\Delta^\star$ and the nonwinners beat one another cyclically with gap $d$, so no total order exists and $\Delta=\min\{\Delta^\star,d\}$. The main weak-regret comparison uses a randomly perturbed version. All online costs include identification. Appendix~\ref{app:time} specifies the full protocol and reports all results.

\paragraph{Comparison.}
On cyclic instances with $(\Delta^\star,d)=(.32,.38)$ and $\delta=.1$ (Figure~\ref{fig:experiments}(a)), the cost of TG-ITE-BAI doubles with $N$, as Corollary~\ref{cor:tlcr-bai} predicts. TG-ITE-BAI uses $10\%$ to $20\%$ more comparisons than BKT. This overhead is expected, since \textsc{TreeAscent} is designed as a shared stage for all three objectives rather than for BAI alone. Knockout needs even fewer comparisons, but it is given the tolerance $\epsilon=.95\Delta^\star$, computed from the true winner gap. On the perturbed instance with $N=256$, $\Delta^\star=.26$, and $T=50{,}000$ (Figure~\ref{fig:experiments}(b)), the weak regret of TG-ITE-Weak stays flat once its warm start identifies the winner, as Theorem~\ref{thm:online-weak-linear-main} predicts, and ends below those of WR-TINF, WR-EXP3-IX, and WS-W with a much smaller spread across seeds. Because its incumbent takes part in every duel, WSW-PE+ETC accrues no weak regret once that incumbent is the winner and ends lower on average, but its spread across seeds is far larger. On the cyclic instance with $N=256$, $(\Delta^\star,d)=(.30,.38)$, and $T=100{,}000$ (Figure~\ref{fig:experiments}(c)), the strong regret of TG-ITE-KH stays flat once self-play begins, as Theorem~\ref{thm:strong-known-main} predicts, and ends below those of all baselines. Versatile-DB comes closest but is still growing at the horizon. Without the horizon, TG-ITE-Anytime repeats identification at increasing confidence levels and ends above WR-TINF and WS-S. It discards past outcomes at the start of every phase, as its analysis assumes. In a paired ablation on this configuration with $T=10^6$ (Table~\ref{tab:fresh-phase-costs}), it uses $2.94$ times as many new identification comparisons as a heuristic that caches pair statistics across phases, since the first phase already identifies the winner in every run and later phases only confirm it.

\paragraph{Dependence on the global gap.}
With $\Delta^\star$ held at $.30$, the mean cost of TG-ITE-BAI grows by a factor of about $4.6$ each time $\Delta$ halves, consistent with the $\Delta^{-2}\log(1/\Delta)$ dependence in Corollary~\ref{cor:tlcr-bai} (Table~\ref{tab:gap-bai-corrected}). TG-ITE-Weak and TG-ITE-KH begin with a \textsc{TreeAscent} pass, so their regret grows at a similar rate. Once $\Delta\le.05$, this pass no longer finishes within $T=10^6$ (Tables~\ref{tab:gap-weak-corrected} and~\ref{tab:gap-strong-corrected}). WR-TINF's guarantee depends on $\Delta^\star$ instead, so the comparison with it hinges on which gap is small. TG-ITE-Weak falls behind it on this sweep once $\Delta\le.20$, but with $\Delta^\star=\Delta=.05$ and nonwinner gap $.20$, it incurs about a tenth as much weak regret (Table~\ref{tab:small-winner-corrected}). Given $g=\Delta^\star$, the known-gap variants of Section~\ref{sec:known-gap-budgeted-extension} cap every test, so their cost and regret level off, as Theorem~\ref{thm:known-gap-extension-main} predicts.

\section{Limitations and Future Work}
\label{sec:future}

Nearly tied nonwinners can dominate the cost of the gap-oblivious policies even when $\Delta^\star$ is large (Proposition~\ref{prop:global-gap-obstruction}, Section~\ref{sec:experiments}). In the Gaussian formulation of \citet[Appendix~A, Theorem~6]{jamieson2014lil}, Farrell's theorem shows that a test with a fixed error level for a single pair with unknown gap $\gamma$ requires $\Omega(\gamma^{-2}\log\log(1/\gamma))$ expected samples along some gaps tending to zero. Thus uniformly correct, gap-oblivious pairwise tests cannot have a uniform $O(\gamma^{-2})$ expected stopping-time bound. Section~\ref{sec:known-gap-budgeted-extension} removes this cost and the $\log(1/\Delta)$ term at the price of knowing a lower bound $0<g\le\Delta^\star$ in advance, so its policies are no longer completely gap-oblivious. Our known-gap experiments supply $g=\Delta^\star$ as an oracle input, and learning such a bound online while keeping these guarantees remains open. Finally, the strong-regret bounds hold only in expectation, and a rare wrong output followed by a long self-play block can produce regret near the horizon (Appendix~\ref{app:experiment-tail}).
We list three directions for future work.
\textit{(i)} For a known horizon $H$, we conjecture that relaxing the BAI objective in \textsc{TreeAscent} to a coarse capture pass followed by direct incumbent-versus-challenger verification attains gap-weighted strong regret $O(N\Delta^{-2}\log(1/\Delta))+O(\sum_{i\ne a^\star}\log H/\Delta_i^\star)$. Proving this requires controlling the cost incurred while the incumbent is incorrect.
\textit{(ii)} Caching statistics across phases of the strong-regret policy saves substantial cost in Section~\ref{sec:experiments} but has no guarantee yet. A martingale-based analysis could provide one.
\textit{(iii)} Without gap information, iterated-logarithm confidence sequences \citep{howard2021time} could reduce the $\log(1/\Delta)$ term to $\log\log(1/\Delta)$ if the error and expected stopping-time bounds of \textsc{DuelTest} carry over.

\section*{Acknowledgement}
This work was supported by National Natural Science Foundation of China (U23A20311, 62206245). The authors declare no competing interests. We would like to thank the anonymous reviewers for their constructive suggestions, as well as Yu-Hu Yan, Peng Zhao, and Fang Kong for insightful discussions.

\bibliographystyle{named}
\bibliography{ref}

\appendix
\section{Related Work}
\label{sec:rw}

This section reviews the work most related to our setting. We group prior results by objective: regret minimization, pure exploration, and attempts to unify multiple objectives.

\subsection{Regret minimization for dueling bandits}
\label{rw:rm}

Dueling bandits were introduced to model online learning from noisy pairwise preferences, with early applications in interactive information retrieval and online ranker evaluation \citep{yue2009interactively,yue2012k}. The regret-minimization literature has since developed along several assumptions and benchmarks. Early algorithms such as Interleaved Filter and Beat-the-Mean rely on latent-order or stochastic-transitivity-type assumptions, and achieve logarithmic-in-$T$ regret under such structure \citep{yue2011beat,yue2012k}. Later work relaxed or changed the structural requirements: \citet{ailon2014reducing} reduce dueling bandits to conventional cardinal bandits; \citet{zoghi2014relative} propose RUCB for the Condorcet setting and prove an $O(N\log T)$ finite-time regret bound; \citet{zoghi2015mergerucb} introduce MergeRUCB for large-scale online ranker evaluation; and \citet{komiyama2015regret} give an asymptotic lower bound based on information divergence together with RMED, the first asymptotically matching algorithm under the Condorcet assumption. Other works study alternative winner notions or extensions, including Copeland winners, Thompson-sampling-style methods, contextual dueling bandits, multi-dueling bandits, dependent-arm models, and batched feedback \citep{zoghi2015copeland,wu2016double,dudik2015contextual,sui2017multidueling,du2020dueling,agarwal2022batched}. These directions are mostly orthogonal to our finite-arm Condorcet setting.

\paragraph{Weak regret.}
Weak regret counts a round as regret-free whenever at least one selected arm is the Condorcet winner. This objective can be much easier than strong regret because one coordinate can exploit the current incumbent while the other coordinate explores challengers. \citet{chen2017dueling} made this distinction explicit and proposed Winner-Stays-Weak (WS-W), whose expected cumulative weak regret is $O(N^2)$ under the Condorcet-winner assumption and improves to $O(N\log N)$ when the arms follow a total order. Their Theorem~1 gives the explicit total-order bound $O(N\Delta^{-6}\log N)$, where $\Delta$ is a minimum pairwise gap. More recently, \citet{saad2024weak} showed that weak regret under the sole Condorcet-winner assumption is governed by the full preference matrix: they prove an $\Omega(N/\Delta^\star)$ lower bound in the large-optimality-gap regime, where $\Delta^\star:=\min_{i\neq a^\star}(p_{a^\star,i}-1/2)$, and propose WR-TINF, which matches this order in that regime; they also give another algorithm for regimes where the Condorcet gap is not the right hardness parameter. The closest joint RM--BAI work is \citet{ma2025dual}, which builds a reduction from weak-regret minimization to fixed-confidence BAI; instantiated with WS-W, their WSW-PE algorithm preserves the weak-regret guarantee $O(N\Delta^{-6}\log N)$ and identifies the Condorcet winner with high probability using
$O\!\left( R_w + \frac{N}{(\Delta^\star)^2} \log \frac{N}{\Delta^\star\delta} \right)$
comparisons. In contrast, our TG-ITE-Weak keeps the Winner-Stays-style incumbent/challenger structure but replaces the online discovery phase with a tree-guided warm start and screening tail. Under the positive global-gap assumption $\Delta:=\min_{i\neq j}|p_{i,j}-1/2|$, it satisfies
\[
\mathbb E[R^{\rm weak}_\infty]
\le
O\!\left(
\frac{N}{\Delta^2}
\left(\log\frac1\Delta+1\right)
\right).
\]
Thus our weak-regret guarantee is time-independent and linear in $N$ without the $\log N$ factor of WS-W/WSW-PE. Compared with \citet{saad2024weak}, our uniform bound depends on the global minimum gap $\Delta$ rather than the instance-dependent optimal gap $\Delta^\star$, which makes it coarser in its gap dependence. However, this is the cost for a simple, reusable identification primitive that supports BAI, weak regret, and strong regret simultaneously. This global-gap parameterization also follows the convention of the classical noisy-comparison literature~\citep{feige1994computing} and of the Winner-Stays line of work~\citep{chen2017dueling}, which state guarantees in terms of a gap governing all pairwise comparisons rather than only those involving the winner.

\paragraph{Strong regret.}
Strong regret requires both selected coordinates to be the Condorcet winner. This is the geometry captured by the standard Condorcet regret used in much of the dueling-bandit literature, where pulling a suboptimal arm in either coordinate incurs gap-weighted loss. \citet{yue2012k} prove an $\Omega(N\log T)$ worst-case expected-regret lower bound under the Condorcet-winner assumption, and Interleaved Filter achieves $O(N\log T)$ under strong stochastic transitivity. Beat-the-Mean further gives a high-probability regret bound of order
$O\!\left(\sum_{i\neq a^\star}\frac{\gamma^8}{\Delta_i}\log T\right)$, $\Delta_i:=p_{a^\star,i}-1/2$
under relaxed stochastic transitivity and stochastic triangle inequality \citep{yue2011beat}. RUCB obtains an $O(N\log T)$ finite-time bound in the Condorcet setting \citep{zoghi2014relative}, while RMED is asymptotically optimal without additional structural assumptions \citep{komiyama2015regret}. \citet{pmlr-v162-saha22a} later give a reduction-based best-of-both-worlds algorithm, Versatile-DB, which is optimal for both stochastic and adversarial preferences and, in the stochastic Condorcet case, achieves the gap-dependent bound
$ O\!\left( \sum_{i\neq a^\star}\frac{\log T}{\Delta_i} \right) $
against the Condorcet-winner benchmark. In the binary strong-regret formulation of \citet{chen2017dueling}, WS-S obtains $O(N^2+N\log T)$ expected cumulative strong regret, improved to $O(N\log N+N\log T)$ under total order. Our strong-regret policy instead identifies a reliable incumbent and self-plays. With known horizon $T$, it guarantees
\[
\mathbb E[R^{\rm str}_T]
\le
O\!\left(
\frac{N}{\Delta^2}
\left(\log T+\log\frac1\Delta+1\right)
\right),
\]
and the horizonless version has the same leading $O(N\Delta^{-2}\log T)$ term up to lower-order $\log\log T$ factors. Therefore, our strong-regret result removes the $N^2$ exploration term of WS-S and matches the linear-in-$N$, logarithmic-in-$T$ behavior of specialized strong-regret algorithms, while using the same \textsc{TreeAscent} primitive as the BAI and weak-regret policies.

\subsection{Pure Exploration for Dueling Bandits}
\label{rw:bai}
Pure exploration has a long history in ordinary stochastic bandits, where best-arm identification is separated from cumulative regret minimization. In the standard stochastic MAB setting, fixed-confidence and fixed-budget BAI have been extensively studied, with sample complexity governed by inverse squared reward gaps, such as $\sum_{i\neq i^\star}\Delta_i^{-2}$, up to logarithmic factors in the confidence parameter \citep{evendar2006action,audibert2010best,gabillon2012best,kaufmann2016complexity}. Dueling bandits replace numerical rewards by noisy pairwise comparisons, so pure exploration first requires specifying what the ``best'' arm means. Common winner notions include Condorcet, Borda, and Copeland winners; in this work we focus on the Condorcet winner \citep{bengs2021preference}.

A large body of pure-exploration work for dueling bandits studies winner identification under structural assumptions on the preference matrix. Let $\Delta_{i,j}=p_{i,j}-1/2$ denote the signed pairwise gap. Weak stochastic transitivity (WST) requires 
\[
\Delta_{i,j}\ge 0,\ \Delta_{j,k}\ge 0
\quad\Longrightarrow\quad
\Delta_{i,k}\ge 0,
\]
which induces a transitive preference relation and hence, up to ties or tie-breaking, a latent total order. Strong stochastic transitivity (SST) is stronger: for ordered arms $i\succ j\succ k$, it requires
\[
\Delta_{i,k}\ge
\max\{\Delta_{i,j},\Delta_{j,k}\}.
\]
Thus SST not only implies a latent total order but also imposes monotonicity on comparison margins. This distinction matters statistically. 

Under total-order-type assumptions, \citet{yue2011beat} proposed BTM-PAC, which identifies an $\epsilon$-optimal arm with sample complexity
$O\!\left( \frac{N}{\epsilon^2} \log\frac{N}{\epsilon\delta} \right).$
Subsequent noisy-comparison work sharpened this dependence under stronger structural assumptions. In particular, \citet{falahatgar2017maximum} gave a knockout-style maximum-selection algorithm under SST and stochastic triangle inequality (STI) with
$O\!\left( \frac{N}{\epsilon^2} \left(\log\frac1\delta\right) \right)$
comparisons, which is optimal up to constants for PAC maxing. Later work clarified how much structure is needed: \citet{falahatgar2017few} showed that SST alone is enough for linear-complexity PAC maxing, while \citet{falahatgar2018limits} showed that WST alone can require $\Omega(N^2)$ comparisons for general PAC maxing, even though WST still induces a latent order.

Another line studies exact identification rather than approximate maxing. Exact Condorcet identification can be viewed as the regime $0<\epsilon<\Delta^\star$, where
$\Delta^\star := \min_{j\neq a^\star} |p_{a^\star,j}-1/2|$
is the Condorcet minimum gap. Classical noisy-comparison results analyze selection when each comparison errs with a fixed probability~\citep{feige1994computing}. On instances in which every pairwise gap equals $\Delta$, any $\delta$-correct algorithm needs
$\Omega\!\left( \frac{N}{\Delta^2} \log \frac1\delta \right)$
comparisons, even when $\Delta$ is known (Proposition~\ref{prop:bai-homogeneous-lower-bound}). Under parametric or ranking-style assumptions, \citet{busaFekete2014mallows} proposed MallowsMPI for exact BAI under a Mallows model, with sample complexity
$
O\!\left(
\frac{N}{\Delta^2}
\log\frac{N}{\delta\Delta}
\right).$
\citet{mohajer2017active} studied active top-$K$ rank aggregation from noisy comparisons using fixed-repetition tournaments with comparison budgets calibrated to the gaps. More recent exact-ranking and best-$k$ selection results further refine this picture: \citet{ren2019sample} characterize exact ranking from noisy comparisons, and \citet{ren2020sample} study PAC and exact best-$k$ selection under SST and STI, obtaining matching PAC best-$k$ bounds of order
$\Theta\!\left(\frac{N}{\epsilon^2}\log\frac{k}{\delta}\right)$ for $1\le k\le N/2$
and exact top-$1$ guarantees that are optimal up to logarithmic or doubly-logarithmic factors.

The closest setting to ours is Condorcet winner identification without requiring a latent total order among all arms. \citet{haddenhorst2021testification} study statistical procedures for testing and identifying Condorcet winners. \citet{maiti2024near} study pure exploration in matrix games, a framework that subsumes stochastic bandits and dueling bandits, give near-optimal guarantees for identifying pure-strategy Nash equilibria, and obtain a high-probability Condorcet-identification bound governed by winner-to-arm gaps. Most recently, \citet{saad2026sampling} introduce an identification procedure that exploits the full gap matrix $\Delta_{i,j}=p_{i,j}-1/2$ rather than only the winner-to-arm gaps $\{\Delta_{a^\star,j}\}$, leveraging informative comparisons among non-winners to improve, up to logarithmic factors, the best known instance-dependent sample-complexity rates. They further establish the first lower bounds for Condorcet identification under the sole Condorcet-winner assumption, matching their upper bounds up to logarithmic factors and isolating the intrinsic cost of locating and estimating informative entries of the gap matrix; this reveals new regimes in which the optimal complexity is governed by the full preference matrix rather than by $\Delta^\star$ alone. 

Our BAI guarantee uses a global-gap accounting with $\Delta := \min_{i\neq j}|p_{i,j}-1/2|$ 
and obtain a $\delta$-correct identification primitive with expected sample complexity
\[
O\!\left(
\frac{N}{\Delta^2}
\left(
\log\frac1\delta+\log\frac1\Delta+1
\right)
\right).
\]
Compared with early BTM-PAC or PE-style bounds containing $\log(N/(\epsilon\delta))$ or $\log(N/(\Delta\delta))$, the confidence dependence in our bound is $\log(1/\delta)$ rather than $\log(N/\delta)$; the remaining $\log(1/\Delta)$ term comes from the sequential pairwise tests used inside the tournament primitive. Compared with SST/STI-based maxing algorithms, our procedure does not rely on a total order or magnitude-transitivity assumptions among all arms. Compared with the instance-dependent CW-identification guarantees of \citet{saad2026sampling}, our bound is coarser in the gap structure, but the resulting primitive is simple, starting-arm-uniform, and directly reusable in the weak- and strong-regret components of our framework.

\subsection{Targeting Multi-Objective Unification}

The relation between regret minimization and best-arm identification is well studied in ordinary bandits. Fixed-confidence identification requires enough exploration to rule out alternatives, while regret minimization rewards early exploitation. This difference is reflected in classical BAI algorithms and lower bounds \citep{audibert2010best,kaufmann2016complexity}. Explore-then-commit strategies make the tension explicit and can be suboptimal for regret minimization \citep{garivier2016etc}. Several works study the interface more directly, including algorithms for A/B testing and Pareto-frontier analyses of regret minimization versus best-arm identification \citep{degenne2019bridging,zhong2022pareto}.

In dueling bandits, the learner chooses two arms per round, which creates additional ways to combine exploration and exploitation. Winner-Stays-style algorithms use one coordinate as an incumbent and the other to test challengers, which is especially natural for weak regret \citep{chen2017dueling}. The work closest to our multi-objective motivation is \citet{ma2025dual}, which studies joint regret minimization and sample complexity for dueling bandits and proposes a black-box reduction transforming any regret-minimization algorithm into a best-arm identification algorithm. However, their proposed algorithm only achieves the WS-W style weak regret bound of order $O(N \log N)$, a $\log N$ factor above the $O(N)$ optimum of \citet{saad2024weak}. And their algorithm is a further $\log N$ factor suboptimal for BAI. 

\section{Primitive Tests and Tournament Subroutines}
\label{app:primitives}

This appendix collects the statistical primitives used throughout the paper and gives their complete proofs. We keep the statements in a global-gap form because this is the accounting used by the main results. Recall that
\[
    \Delta_{i,j}:=\left|p_{i,j}-\frac12\right|,
    \qquad
    \Delta:=\min_{i\neq j}\Delta_{i,j}.
\]
All logarithms are natural. Universal numerical constants may change from line to line.

All guarantees below hold conditionally on the history before a fresh call, which fixes the arms, candidate set, tree, and confidence argument. Given this history, the new observations of each selected pair are i.i.d. Bernoulli with the prescribed mean. Applying the pairwise bounds test by test and averaging over their entry histories therefore extends them to the whole subroutine. This allows the online policies to choose their inputs adaptively.

\subsection{\textsc{DuelTest}}
\label{app:dueltest}

\begin{lemma}[Anytime confidence sequence]
\label{lem:app-anytime-cs}
Let $Y_1,Y_2,\ldots$ be i.i.d. Bernoulli random variables with mean $p$, and let $\widehat\mu_n=n^{-1}\sum_{t=1}^nY_t$.  Then for every $\delta\in(0,1)$,
\[
    \Pr\left(\exists n\ge 1:
    |\widehat\mu_n-p|\ge c(n,\delta)
    \right)
    \le \delta .
\]
\end{lemma}

\begin{proof}
For any fixed $n\ge1$, Hoeffding's inequality gives
\[
    \Pr\{|\widehat\mu_n-p|\ge c(n,\delta)\}
    \le
    2\exp\{-2n c(n,\delta)^2\}
    =
    \frac{\delta}{n(n+1)}.
\]
Taking a union bound over all $n\ge1$ and using $\sum_{n\ge1}1/(n(n+1))=1$ proves the claim.
\end{proof}

\begin{lemma}[Correctness of \textsc{DuelTest}]
\label{lem:app-dueltest-correct}
For every pair $u\neq v$ with $p_{u,v}\neq1/2$, Algorithm~\ref{alg:dueltest} terminates almost surely and returns the better arm with probability at least $1-\delta$.
\end{lemma}

\begin{proof}
Almost-sure termination follows from the strong law of large numbers: $\widehat\mu_n\to p_{u,v}$ while $c(n,\delta)\to0$, so the stopping inequality eventually holds when $p_{u,v}\neq1/2$. We prove the case $p_{u,v}=1/2+\gamma$ with $\gamma>0$. The other case is symmetric. Let
\[
    \mathcal E_\delta
    :=
    \{\forall n\ge1:
      |\widehat\mu_n-p_{u,v}|<c(n,\delta)\}.
\]
By Lemma~\ref{lem:app-anytime-cs}, $\Pr(\mathcal E_\delta)\ge1-\delta$.  On $\mathcal E_\delta$, the test cannot stop and return $v$.  Indeed, if it returned $v$ at time $\tau$, then $\widehat\mu_\tau<1/2-c(\tau,\delta)$ by the stopping rule and the decision rule.  Hence
\[
    p_{u,v}-\widehat\mu_\tau
    >
    \left(\frac12+\gamma\right)
    -
    \left(\frac12-c(\tau,\delta)\right)
    =
    \gamma+c(\tau,\delta)
    > c(\tau,\delta),
\]
contradicting $\mathcal E_\delta$.  Thus the error event is contained in $\mathcal E_\delta^c$, whose probability is at most $\delta$.
\end{proof}

The expected number of comparisons made by \textsc{DuelTest} has the standard inverse-gap form.  We record a version with a concrete threshold.

\begin{lemma}[Stopping-time bound for one pair]
\label{lem:app-dueltest-time}
There is a universal constant $C_{\rm dt}>0$ such that, for every pair $(u,v)$ with gap $\gamma=|p_{u,v}-1/2|\in(0,1/2]$ and every $\delta\in(0,1)$, the stopping time $\tau(u,v,\delta)$ of Algorithm~\ref{alg:dueltest} satisfies
\begin{equation}
\label{eq:app-dueltest-exp}
    \mathbb E[\tau(u,v,\delta)]
    \le
    \frac{C_{\rm dt}}{\gamma^2}
    \left(
        \log\frac1\delta+
        \log\frac1\gamma+1
    \right).
\end{equation}
In particular, since $\gamma\ge\Delta$ for every distinct pair and $\Delta>0$ under Assumption~\ref{assump:condorcet-gap},
\[
    \mathbb E[\tau(u,v,\delta)]
    \le
    \frac{C_{\rm dt}}{\Delta^2}
    \left(
        \log\frac1\delta+
        \log\frac1\Delta+1
    \right).
\]
\end{lemma}

\begin{proof}
It is enough to consider $p_{u,v}=1/2+\gamma$. The other case is symmetric.  Define
\[
    n_\star(\gamma,\delta)
    :=
    \left\lceil
    \frac{C_0}{\gamma^2}
    \left(
        \log\frac1\delta+
        \log\frac1\gamma+1
    \right)
    \right\rceil,
\]
where $C_0=32$. Put
\[
    A:=\log\frac1\delta+\log\frac1\gamma+1\ge1.
\]
The ceiling and $\gamma\le1/2$ imply
$n_\star\le64A/\gamma^2$. Using $\log A\le A-1$, we obtain
\begin{align*}
    \log\frac4\delta+2\log n_\star
    &\le \log4+2\log64+\log\frac1\delta
       +4\log\frac1\gamma+2\log A\\
    &\le 6A+(\log4+2\log64-6)
     \le 16A
     \le \frac{n_\star\gamma^2}{2}.
\end{align*}
The function $n\mapsto n\gamma^2/2-2\log n-\log(4/\delta)$ is nondecreasing for $n\ge4/\gamma^2$, and $n_\star$ exceeds this threshold. Therefore, for every $n\ge n_\star$, the bound $n+1\le2n$ gives
\[
    \log\bigl(2n(n+1)/\delta\bigr)
    \le \log\frac4\delta+2\log n
    \le \frac{n\gamma^2}{2}.
\]
It follows that $c(n,\delta)\le\gamma/2$ throughout this range.

If the test has not stopped by time $n\ge n_\star(\gamma,\delta)$, then $|\widehat\mu_n-1/2|\le c(n,\delta)\le\gamma/2$.  Therefore
\[
    \{\tau>n\}
    \subseteq
    \left\{\widehat\mu_n\le \frac12+\frac\gamma2\right\}
    =
    \left\{\widehat\mu_n-p_{u,v}\le -\frac\gamma2\right\}.
\]
A one-sided Hoeffding bound yields
\[
    \Pr(\tau>n)
    \le
    \exp\{-n\gamma^2/2\},
    \qquad n\ge n_\star(\gamma,\delta).
\]
Using the tail-sum formula,
\[
    \mathbb E[\tau]
    =
    \sum_{n\ge0}\Pr(\tau>n)
    \le
    n_\star(\gamma,\delta)
    +
    \sum_{n\ge n_\star(\gamma,\delta)}e^{-n\gamma^2/2}.
\]
Since $1-e^{-x}\ge x/2$ for $x\in(0,1]$ and $\gamma\le1/2$, the geometric tail is at most $4/\gamma^2$.  Absorbing constants into $C_{\rm dt}$ gives~\eqref{eq:app-dueltest-exp}.
\end{proof}

The test needs a positive gap to terminate. If $p_{u,v}=1/2$, the confidence event of Lemma~\ref{lem:app-anytime-cs} keeps $|\widehat\mu_n-1/2|<c(n,\delta)$ for every $n$, so the test runs forever with probability at least $1-\delta$. \textsc{BudgetedDuelTest} (Section~\ref{sec:known-gap-budgeted-extension}) stops within its budget even on such ties.

\subsection{\textsc{BKT}}
\label{app:bkt}

\begin{lemma}[BKT on an arbitrary candidate set]
\label{lem:app-bkt-subset}
Under Assumption~\ref{assump:condorcet-gap}, let $S\subseteq[N]$ be nonempty and let $m=|S|$. For every $\delta\in(0,1)$, Algorithm~\ref{alg:bkt} satisfies the following two properties.
First, if $m\ge2$, its expected number of comparisons satisfies
\begin{equation}
\label{eq:app-bkt-cost}
    \mathbb E[\tau_{\rm BKT}(S,\delta)]
    \le
    \frac{C_{\rm bkt}m}{\Delta^2}
    \left(
        \log\frac1\delta+
        \log\frac1\Delta+1
    \right)
\end{equation}
for a universal constant $C_{\rm bkt}>0$.  Second, suppose that $S$ contains an arm $w$ satisfying
\[
    p_{w,j}>\frac12
    \qquad\text{for every }j\in S\setminus\{w\}.
\]
Then Algorithm~\ref{alg:bkt} returns $w$ with probability at least $1-\delta$. In particular, if $a^\star\in S$, then \textsc{BKT}$(S,\delta)$ returns $a^\star$ with probability at least $1-\delta$.
\end{lemma}

\begin{proof}
If $m=1$, the algorithm returns the unique arm in $S$, so the correctness claim is immediate whenever the displayed condition holds.  For $m\ge2$, assume first that such an arm $w$ exists.  Track only $w$ through the bracket. In each tournament round $r$, this arm participates in at most one call to \textsc{DuelTest}, with confidence $\delta_r=6\delta/(\pi^2r^2)$. If $w$ is ever eliminated, one of these calls must return an arm that is worse than $w$, and hence must be wrong. By Lemma~\ref{lem:app-dueltest-correct} and a union bound,
\[
    \Pr\{w\text{ is eliminated}\}
    \le
    \sum_{r\ge1}\delta_r
    =
    \sum_{r\ge1}\frac{6\delta}{\pi^2r^2}
    =\delta.
\]
If $w$ is never eliminated, it is the final remaining arm.

We now prove the expected-cost bound, which does not require $S$ to contain any Condorcet arm. Let $X_{r-1}$ denote the active set at the beginning of round $r$, and let $K_r$ be the number of pairwise tests in that round. The halving bound is deterministic: regardless of the duel outcomes, the bracket keeps at most one survivor from each pair and possibly one bye, so
\[
    |X_{r-1}|\le \left\lceil \frac{m}{2^{r-1}}\right\rceil,
    \qquad
    K_r\le \left\lfloor\frac{|X_{r-1}|}{2}\right\rfloor .
\]
The number of nonempty rounds is at most $R:=\lceil\log_2 m\rceil$. By Lemma~\ref{lem:app-dueltest-time}, each test in round $r$ has expected length at most
\[
    \frac{C}{\Delta^2}
    \left(
        \log\frac1\delta+2\log r+
        \log\frac1\Delta+1
    \right).
\]
Therefore
\[
\mathbb E[\tau_{\rm BKT}(S,\delta)]
\le
\frac{C}{\Delta^2}
\sum_{r=1}^{R}K_r
\left(
        \log\frac1\delta+2\log r+
        \log\frac1\Delta+1
\right).
\]
The total number of matches is at most $m-1$, so the terms not involving $\log r$ contribute at most the right-hand side of~\eqref{eq:app-bkt-cost}. For the remaining term, the deterministic halving bound gives
\[
    \sum_{r=1}^{R}K_r\log r
    \le
    \sum_{r=1}^{R}\left(\frac{m}{2^r}+1\right)\log r
    \le
    C m + C R\log R.
\]
Since $R\le\lceil\log_2m\rceil$ and $R\log R\le C m$ for all $m\ge2$, the logarithmic-round contribution is also $O(m)$. This proves Equation~\eqref{eq:app-bkt-cost} after increasing the constant.
\end{proof}

When $S$ contains no Condorcet arm, \textsc{BKT} should be interpreted only as returning a bracket representative. The expected-cost part of Lemma~\ref{lem:app-bkt-subset} still applies to such sets.

\section{Tree-Guided Identification}
\label{app:tlr-proofs}

This appendix gives the complete proof of the tree-guided identification bound, including the formal-bye treatment for arbitrary $N\ge2$.
The guarantees for \textsc{TreeAscent} use Assumption~\ref{assump:condorcet-gap} and, as in Appendix~\ref{app:primitives}, hold conditionally on the history before the pass.

\subsection{Tree-Based Tournament Decomposition}
\label{app:tree-geometry}

If $N$ is not a power of two, embed the $N$ real arms into a completed balanced bracket with
\[
    M:=2^{\lceil\log_2N\rceil}<2N
\]
leaves. The additional leaves are formal byes: they are not arms, are not assigned preference probabilities, and never generate stochastic comparisons. If $N$ is a power of two, $M=N$. Let $L:=\log_2M$ and let $\mathcal T$ be the completed balanced tree. For a node $v$, write $\bar S(v)$ for all bracket leaves under $v$ and $S(v):=\bar S(v)\cap[N]$ for the real arms under $v$. We do not define a local maximum for arbitrary nonempty $S(v)$, because such a maximum need not exist when nonwinner arms form cycles.

For a starting real arm $b$, let
\[
    v_0(b)\subset v_1(b)\subset\cdots\subset v_L(b)=\mathrm{root}
\]
be its ancestor chain in the completed tree, with $\bar S(v_0(b))=\{b\}$ and $|\bar S(v_h(b))|=2^h$. Write $U_h(b):=S(v_h(b))$. Let $q_h(b)$ be the sibling of $v_{h-1}(b)$ inside $v_h(b)$ and define
\[
    Q_h(b):=S(q_h(b))=U_h(b)\setminus U_{h-1}(b),
    \qquad
    \bar Q_h(b):=\bar S(q_h(b)).
\]
The set $Q_h(b)$ may be empty because formal byes are ignored.

\begin{lemma}[Sibling-block partition with formal byes]
\label{lem:app-sibling-partition}
For every real arm $b\in[N]$,
\[
    [N]=\{b\}\sqcup Q_1(b)\sqcup Q_2(b)\sqcup\cdots\sqcup Q_L(b),
\]
where the union is disjoint. If $m_h:=|Q_h(b)|$ and $\bar m_h:=|\bar Q_h(b)|$, then
\[
    \sum_{h=1}^L m_h=N-1,
    \qquad
    \bar m_h=2^{h-1},
    \qquad
    \sum_{h=1}^L\bar m_h=M-1.
\]
In particular, when $N=2^L$, $M=N$ and $m_h=\bar m_h=2^{h-1}$.
\end{lemma}

\begin{proof}
For every $h\ge1$, the node $v_h(b)$ has two children: $v_{h-1}(b)$ and $q_h(b)$. Restricting to real arms gives
\[
    U_h(b)=U_{h-1}(b)\sqcup Q_h(b).
\]
Iterating this identity up to $h=L$ gives the real-arm partition because $U_0(b)=\{b\}$ and $U_L(b)=[N]$. This also gives $\sum_hm_h=N-1$. For capacities, the completed bracket is full and balanced, so $|\bar Q_h(b)|=|\bar S(v_{h-1}(b))|=2^{h-1}$ and $\sum_{h=1}^L2^{h-1}=M-1$.
\end{proof}

\begin{definition}[Winner-entry level]
\label{def:app-winner-entry-level}
For $b\in[N]$, define
\[
    h_\star(b)
    :=
    \begin{cases}
    0, & b=a^\star,\\
    \text{the unique }h\in\{1,\ldots,L\}\text{ such that }a^\star\in Q_h(b), & b\neq a^\star .
    \end{cases}
\]
The second case is well-defined by the sibling-block partition in Lemma~\ref{lem:app-sibling-partition}.
\end{definition}

\subsection{Correctness of \textsc{TreeAscent}}

\label{app:plr-correctness}

\begin{algorithm}[h]
\caption{\textsc{TreeAscent}$(b,\mathcal T,\varepsilon)$: tree-guided identification}
\label{alg:app-plr}
\begin{algorithmic}[1]
\REQUIRE Starting arm $b\in[N]$, completed balanced tree $\mathcal T$ with $M$ leaves, confidence $\varepsilon\in(0,1)$.
\STATE Let $L\leftarrow\log_2M$. Compute the ancestor chain $v_0(b)\subset\cdots\subset v_L(b)$ and the sibling blocks $Q_h(b)$ and $\bar Q_h(b)$.
\STATE $c_0\leftarrow b$.
\FOR{$h=1,\ldots,L$}
    \STATE $m_h\leftarrow |Q_h(b)|$ and $\bar m_h\leftarrow|\bar Q_h(b)|$.
    \STATE Set $\beta_h\leftarrow (\varepsilon/2)\bar m_h/(M-1)$ and $\eta_h\leftarrow \varepsilon/(2L)$.
    \IF{$m_h=0$}
        \STATE $c_h\leftarrow c_{h-1}$.
    \ELSE
        \STATE $g_h\leftarrow\textsc{BKT}(Q_h(b),\beta_h)$.
        \STATE $c_h\leftarrow\textsc{DuelTest}(c_{h-1},g_h,\eta_h)$.
    \ENDIF
\ENDFOR
\STATE \textbf{return} $c_L$.
\end{algorithmic}
\end{algorithm}

\begin{lemma}[Capture and preservation of the Condorcet winner]
\label{lem:app-plr-capture-preserve}
Fix a starting arm $b$ and let $h_\star=h_\star(b)$ be the winner-entry level from Definition~\ref{def:app-winner-entry-level}. Suppose that, if $h_\star>0$, the call \textsc{BKT}$(Q_{h_\star}(b),\beta_{h_\star})$ returns $a^\star$. Suppose also that every invoked merge call at levels $h\ge \max\{1,h_\star\}$ returns $a^\star$ whenever $a^\star$ is one of its two inputs. Then Algorithm~\ref{alg:app-plr} returns $a^\star$.
\end{lemma}

\begin{proof}
If $h_\star=0$, then $c_0=a^\star$. Consider any later level $h$. If $Q_h(b)$ is empty, the algorithm sets $c_h=c_{h-1}=a^\star$. If $Q_h(b)$ is nonempty, the merge duel compares $a^\star$ with the representative $g_h$, and the stated merge condition returns $a^\star$. Thus $c_h=a^\star$ for all $h$.

If $h_\star>0$, we make no claim about the champion before level $h_\star$. Since $a^\star\in Q_{h_\star}(b)$, this block is nonempty. By the stated BKT condition, $g_{h_\star}=a^\star$. The merge duel at level $h_\star$ then compares $a^\star$ with the previous champion and returns $a^\star$, so $c_{h_\star}=a^\star$. The preservation argument from the previous paragraph applies from level $h_\star+1$ onward, and hence $c_L=a^\star$.
\end{proof}

\begin{theorem}[Correctness of one identification pass]
\label{thm:app-plr-correctness}
For every starting arm $b$ and every $\varepsilon\in(0,1)$,
\[
    \Pr\{\textsc{TreeAscent}(b,\mathcal T,\varepsilon)=a^\star\}
    \ge 1-\varepsilon.
\]
\end{theorem}

\begin{proof}
Let $h_\star=h_\star(b)$. If $h_\star>0$, then $a^\star\in Q_{h_\star}(b)$, and $a^\star$ is a Condorcet arm within this block. By Lemma~\ref{lem:app-bkt-subset},
\[
    \Pr\{\textsc{BKT}(Q_{h_\star}(b),\beta_{h_\star})\neq a^\star\}
    \le \beta_{h_\star}.
\]
If $h_\star=0$, no capture tournament is needed.

For each level $h$, define $F_h$ to be the event that the merge call at level $h$ is invoked with $a^\star$ as one of its two inputs but returns the other arm. Conditional on the history before that fresh merge test, Lemma~\ref{lem:app-dueltest-correct} gives failure probability at most $\eta_h$. Hence $\Pr(F_h)\le\eta_h$. On the complement of the capture failure and of all events $F_h$, Lemma~\ref{lem:app-plr-capture-preserve} implies that the final output is $a^\star$. Therefore,
\[
    \Pr\{\textsc{TreeAscent}(b,\mathcal T,\varepsilon)\neq a^\star\}
    \le
    \mathbf 1\{h_\star>0\}\beta_{h_\star}
    +
    \sum_{h=1}^L\eta_h
    \le
    \sum_{h=1}^L\beta_h+
    \sum_{h=1}^L\eta_h .
\]
By construction and Lemma~\ref{lem:app-sibling-partition},
\[
    \sum_{h=1}^L\beta_h
    =
    \frac\varepsilon2\frac{\sum_h\bar m_h}{M-1}
    =\frac\varepsilon2,
    \qquad
    \sum_{h=1}^L\eta_h=\frac\varepsilon2.
\]
This proves the theorem.
\end{proof}

\subsection{Cost of \textsc{TreeAscent}}
\label{app:plr-cost}

\begin{theorem}[One-pass identification cost]
\label{thm:app-plr-cost}
There is a universal constant $C_{\rm TA}>0$ such that for every starting arm $b$ and every $\varepsilon\in(0,1)$, the number of comparisons $\tau_{\rm TA}(b,\varepsilon)$ used by Algorithm~\ref{alg:app-plr} satisfies
\[
    \mathbb E[\tau_{\rm TA}(b,\varepsilon)]
    \le
    \frac{C_{\rm TA}N}{\Delta^2}
    \left(
        \log\frac1\varepsilon+
        \log\frac1\Delta+1
    \right).
\]
The same bound upper-bounds the weak regret incurred during the pass.
\end{theorem}

\begin{proof}
Let $T_h^{\rm sub}$ be the elementary-comparison count of the subtree tournament at level $h$, and let $T_h^{\rm merge}$ be the count of the merge duel. These quantities are zero if $Q_h(b)$ is empty. Then
\[
    \tau_{\rm TA}
    =
    \sum_{h=1}^LT_h^{\rm sub}
    +
    \sum_{h=1}^LT_h^{\rm merge}.
\]
For the subtree terms, the expected-cost part of Lemma~\ref{lem:app-bkt-subset} gives
\[
    \mathbb E[T_h^{\rm sub}]
    \le
    \frac{C m_h}{\Delta^2}
    \left(\log\frac1{\beta_h}+\log\frac1\Delta+1\right),
    \qquad m_h:=|Q_h(b)|.
\]
Therefore
\[
    \sum_{h=1}^L\mathbb E[T_h^{\rm sub}]
    \le
    \frac{C}{\Delta^2}
    \left(
        \sum_{h=1}^Lm_h\log\frac1{\beta_h}
        +(N-1)\left(\log\frac1\Delta+1\right)
    \right).
\]
Since $m_h\le\bar m_h$, $\sum_h\bar m_h=M-1$, and $\beta_h=(\varepsilon/2)\bar m_h/(M-1)$,
\[
    \sum_{h=1}^Lm_h\log\frac1{\beta_h}
    \le
    \sum_{h=1}^L\bar m_h\log\frac1{\beta_h}
    =
    (M-1)\log\frac2\varepsilon
    +
    \sum_{h=1}^L\bar m_h\log\frac{M-1}{\bar m_h}.
\]
Using $\bar m_h=2^{h-1}$ and $M=2^L$,
\[
    \sum_{h=1}^L\bar m_h\log\frac{M-1}{\bar m_h}
    \le
    \sum_{h=1}^L2^{h-1}(L-h+1)\log2
    =
    M\log2\sum_{j=1}^L\frac{j}{2^j}
    \le 2M\log2.
\]
Thus the subtree contribution is bounded by
\[
    \frac{CM}{\Delta^2}
    \left(\log\frac1\varepsilon+\log\frac1\Delta+1\right).
\]
For the merge terms, Lemma~\ref{lem:app-dueltest-time} gives
\[
    \sum_{h=1}^L\mathbb E[T_h^{\rm merge}]
    \le
    \frac{CL}{\Delta^2}
    \left(
        \log\frac{2L}{\varepsilon}+
        \log\frac1\Delta+1
    \right).
\]
The last display is bounded by the same $M$-linear expression after adjusting the universal constant. Indeed, $L\le M$ and $L\log(2L)\le C M$ for a numerical constant $C$, while $L\log(1/\varepsilon)$ and $L(\log(1/\Delta)+1)$ are dominated by $M\log(1/\varepsilon)$ and $M(\log(1/\Delta)+1)$. Since $M<2N$, the comparison bound follows. Finally, weak regret is at most the number of comparisons, proving the final claim.
\end{proof}

\subsection{A lower bound for best-arm identification}
\label{app:bai-lower-bound}
The bound of Corollary~\ref{cor:tlcr-bai} on the expected stopping time of \textsc{TG-ITE-BAI} contains the term $N\Delta^{-2}\log(1/\delta)$. The next proposition shows that this term cannot be avoided when every pairwise gap equals $\Delta$, even by a learner that knows the gap.

\begin{proposition}[Homogeneous-gap identification lower bound]
\label{prop:bai-homogeneous-lower-bound}
Fix $N\ge2$, $d\in(0,1/4]$, and $\delta\in(0,1/10]$. Let $\mathcal H_N(d)$ be the class of $N$-arm instances that have a Condorcet winner and satisfy $|p_{i,j}-1/2|=d$ for all $i\ne j$. Consider an algorithm, possibly taking $g=d$ as input, that is $\delta$-correct on every instance in $\mathcal H_N(d)$. Then every $P\in\mathcal H_N(d)$ satisfies
\[
    \mathbb E_P[\tau]
    \ge
    \frac{(N-1)\,\operatorname{kl}(1-\delta,\delta)}
         {\operatorname{kl}(1/2-d,1/2+d)}
    \ge c\,\frac{N}{d^2}\log\frac1\delta
\]
for a universal $c>0$, where $\operatorname{kl}$ is the Bernoulli relative entropy.
\end{proposition}

\begin{proof}
Let $a^\star$ be the Condorcet winner of $P$, and let $L_i$ be the set of arms that beat arm $i$ in $P$. For each $i\ne a^\star$, let $P^{(i)}$ be obtained from $P$ by reversing every pair $\{i,j\}$ with $j\in L_i$, so that the probability that $i$ wins rises from $1/2-d$ to $1/2+d$. Arm $i$ then beats every other arm, all pairwise gaps still equal $d$, and the pairs not involving $i$ are unchanged. Hence $P^{(i)}\in\mathcal H_N(d)$.

Let $N_{ij}(\tau)$ be the number of comparisons of the unordered pair $\{i,j\}$ before stopping. If $\mathbb E_P[\tau]=\infty$, there is nothing to prove. Otherwise, consider the event $A=\{\tau<\infty,\ \widehat a=a^\star\}$. By $\delta$-correctness on $P$ and $P^{(i)}$, $\Pr_P(A)\ge1-\delta$ and $\Pr_{P^{(i)}}(A)\le\delta$. Apply the change-of-measure inequality of \citet[Lemma~1]{kaufmann2016complexity}, with the selected pairs as arms, at the bounded stopping time $\tau\wedge t$ to the event $\{\tau\le t,\ \widehat a=a^\star\}$, and let $t\to\infty$. Monotone convergence on the left and lower semicontinuity of $\operatorname{kl}$ on the right give
\begin{align*}
    \sum_{j\in L_i}\mathbb E_P[N_{ij}(\tau)]\,
       \operatorname{kl}(1/2-d,1/2+d)
    &\ge \operatorname{kl}\bigl(\Pr_P(A),\Pr_{P^{(i)}}(A)\bigr)\\
    &\ge \operatorname{kl}(1-\delta,\delta).
\end{align*}
Each unordered pair appears only in the inequality for its losing arm, which is never $a^\star$. Summing over $i\ne a^\star$ and using $\sum_{i\ne a^\star}\sum_{j\in L_i}N_{ij}(\tau)\le\tau$ proves the first bound.

For the second bound, $d\le1/4$ and $\delta\le1/10$ give
\[
    \operatorname{kl}(1/2-d,1/2+d)
    =2d\log\frac{1+2d}{1-2d}\le16d^2,
\]
and $\operatorname{kl}(1-\delta,\delta)\ge\tfrac12\log(1/\delta)$. Since also $N-1\ge N/2$, we can take $c=1/64$.
\end{proof}

Every instance in $\mathcal H_N(d)$ satisfies Assumption~\ref{assump:condorcet-gap} with $\Delta=\Delta^\star=d$, so \textsc{TG-ITE-BAI} meets the hypothesis without knowing $d$. Its upper bound matches the lower bound up to constants and an additional $Nd^{-2}\log(1/d)$ term, which comes from the anytime confidence radius of \textsc{DuelTest} (Appendix~\ref{app:dueltest-lower-duration}). Given $g=d$, the budgeted policy of Theorem~\ref{thm:known-gap-extension-main} also meets the hypothesis, and its budget $CNd^{-2}(\log(1/\delta)+1)$ matches the lower bound up to constants since $\log(1/\delta)>1$.

\section{Weak Regret}
\label{app:weak-proofs}

This appendix proves the strictly linear infinite-horizon weak-regret result. The proofs use fresh samples in every screening and replacement call. Here ``fresh'' means that the elementary outcomes observed inside the current call have not been observed or reused before. The selected pairs may still depend on the past and on earlier outcomes inside the same call. The stochastic model only requires that, conditional on each selected pair and the current history, the newly observed Bernoulli outcome is independent of previous outcomes with the corresponding mean. This is the reason that the analysis remains valid even though the incumbent entering an epoch was selected adaptively from earlier data.

\subsection{\textsc{ScreenAndReplace}}
\label{app:screen-repair}

\begin{lemma}[One-step correction probability]
\label{lem:app-one-step-repair}
Under Assumption~\ref{assump:condorcet-gap}, let \(\rho,\kappa\in(0,1)\), and let \(\mathcal H\) denote the history immediately before a call to Algorithm~\ref{alg:screen-and-replace}. This history includes the incumbent \(c\) already chosen by the algorithm. Let \(c^+\) be the arm returned by the call. Then, almost surely,
\[
    \Pr(c^+\neq a^\star\mid\mathcal H)
    \le
    \mathbf 1\{c\neq a^\star\}(\rho+\kappa)
    +
    \mathbf 1\{c=a^\star\}\kappa.
\]
Consequently,
\[
    \Pr(c^+\neq a^\star\mid c\neq a^\star)\le \rho+\kappa,
    \qquad
    \Pr(c^+\neq a^\star\mid c=a^\star)\le \kappa
\]
whenever the conditioning events have positive probability, and \(\Pr(c^+\neq a^\star)\le\rho+\kappa\) unconditionally.
\end{lemma}

\begin{proof}
Fix any realized pre-call history, including the realized value of \(c\). After this conditioning, the only randomness left is the fresh feedback generated inside the current call. Therefore Lemma~\ref{lem:app-dueltest-correct} and the correctness part of Lemma~\ref{lem:app-bkt-subset} can be applied to the tests and tournaments run inside this call, even though the pairs tested inside the call may be chosen adaptively from earlier fresh feedback in the same call.

If \(c\neq a^\star\), then the screening phase compares \(c\) against \(a^\star\) exactly once. Since \(a^\star\succ c\), \textsc{DuelTest}\((c,a^\star,\rho)\) inserts \(a^\star\) into \(B\) with conditional probability at least \(1-\rho\). Conditional on \(a^\star\in B\), the replacement tournament is run on a set containing \(a^\star\), and Lemma~\ref{lem:app-bkt-subset} implies that it returns \(a^\star\) with conditional probability at least \(1-\kappa\). A union bound gives conditional failure probability at most \(\rho+\kappa\).

If \(c=a^\star\), then either \(B=\emptyset\) and the algorithm returns \(a^\star\) deterministically, or \(B\neq\emptyset\) and the replacement tournament is run on a set containing \(a^\star\). In the latter case the tournament fails with conditional probability at most \(\kappa\). This proves the displayed conditional inequality. The remaining claims follow by taking expectations and conditioning on \(\{c\neq a^\star\}\) or \(\{c=a^\star\}\) when those events have positive probability.
\end{proof}

\begin{lemma}[Expected weak regret in one screen-and-replace epoch]
\label{lem:app-screen-epoch-regret}
Under Assumption~\ref{assump:condorcet-gap}, let \(\rho,\kappa\in(0,1)\), and let \(\mathcal H\) denote the history immediately before a call to Algorithm~\ref{alg:screen-and-replace}. This history includes the incumbent \(c\) already chosen by the algorithm. Let \(W(c,\rho,\kappa)\) be the weak regret incurred during the call. Then, almost surely,
\begin{equation}
\label{eq:app-one-epoch-weak-conditional}
\begin{split}
\mathbb E[W(c,\rho,\kappa)\mid\mathcal H]
\le\frac{CN}{\Delta^2}\Bigg[&
\mathbf 1\{c\neq a^\star\}
\left(\log\frac1\rho+\log\frac1\kappa+\log\frac1\Delta+1\right)\\
&+\rho\left(\log\frac1\kappa+\log\frac1\Delta+1\right)\Bigg].
\end{split}
\end{equation}
Consequently, if \(p=\Pr(c\neq a^\star)\), then
\begin{equation}
\label{eq:app-one-epoch-weak}
\begin{split}
\mathbb E[W(c,\rho,\kappa)]
\le\frac{CN}{\Delta^2}\Bigg[&
 p\left(\log\frac1\rho+\log\frac1\kappa+\log\frac1\Delta+1\right)\\
&+\rho\left(\log\frac1\kappa+\log\frac1\Delta+1\right)\Bigg].
\end{split}
\end{equation}
\end{lemma}

\begin{proof}
Fix any realized pre-call history, including the realized value of \(c\). All comparisons inside the call use fresh feedback, so after conditioning on each selected pair and the current within-call history, the expected stopping-time bounds apply exactly as in the nonadaptive case.

During screening, if \(c=a^\star\), every screening comparison is of the form \((a^\star,j)\) and has zero weak regret. If \(c\neq a^\star\), the comparison against \(a^\star\) itself also has zero weak regret. Nevertheless, we use the simpler upper bound by the total number of screening comparisons. There are at most \(N-1\) such tests, each with confidence \(\rho\), so Lemma~\ref{lem:app-dueltest-time} gives conditional screening contribution at most
\[
    \mathbf 1\{c\neq a^\star\}
    \frac{CN}{\Delta^2}
    \left(\log\frac1\rho+
    \log\frac1\Delta+1\right).
\]
If \(c\neq a^\star\) and the replacement tournament is invoked, its weak regret is bounded by its total number of comparisons. Since the candidate set has size at most \(N\), Lemma~\ref{lem:app-bkt-subset} gives conditional contribution at most
\[
    \mathbf 1\{c\neq a^\star\}
    \frac{CN}{\Delta^2}
    \left(\log\frac1\kappa+
    \log\frac1\Delta+1\right).
\]
On the event \(c=a^\star\), a non-best arm \(j\) enters \(B\) only if \textsc{DuelTest}\((a^\star,j,\rho)\) returns \(j\), which has conditional probability at most \(\rho\). Therefore
\[
    \mathbb E[|B|\mid\mathcal H, c=a^\star]\le (N-1)\rho.
\]
Conditional on the realized beater set \(B\), if \(B\neq\emptyset\), the replacement tournament is run on \(|B|+1\) arms and its weak regret is at most its comparison count $T_{\rm repl}$. This tournament uses fresh samples after the screening phase, so Lemma~\ref{lem:app-bkt-subset} applies to the inner conditional expectation. By the tower property,
\[
\mathbb E[T_{\rm repl}\mid\mathcal H,c=a^\star]
=
\mathbb E\!\left[\mathbb E[T_{\rm repl}\mid B,\mathcal H,c=a^\star]\mid\mathcal H,c=a^\star\right].
\]
The inner expectation is at most
\[
    \frac{C(|B|+1)\mathbf 1\{|B|>0\}}{\Delta^2}
    \left(\log\frac1\kappa+
    \log\frac1\Delta+1\right).
\]
Since \((|B|+1)\mathbf 1\{|B|>0\}\le2|B|\), taking expectation over \(B\) gives a false-positive contribution at most
\[
    \frac{CN\rho}{\Delta^2}
    \left(\log\frac1\kappa+
    \log\frac1\Delta+1\right).
\]
Adding the three conditional contributions proves~\eqref{eq:app-one-epoch-weak-conditional}. Taking expectations proves~\eqref{eq:app-one-epoch-weak}.
\end{proof}

\subsection{\textsc{TG-ITE-Weak}}
\label{app:online-weak-algorithm}

\begin{theorem}[Tail bound for infinite-horizon weak regret]
\label{thm:app-online-weak}
Under Assumption~\ref{assump:condorcet-gap}, fix any starting incumbent \(c_0\) chosen before the online tail starts, and run the \textsc{ScreenAndReplace} epochs with \(\rho_s=\kappa_s=4^{-(s+2)}\). If \(R_{\mathrm{tail},\infty}^{\rm weak}\) denotes the weak regret incurred by these epochs only, then
\[
    \mathbb E[R_{\mathrm{tail},\infty}^{\rm weak}]
    \le
    \frac{CN}{\Delta^2}
    \left(\log\frac1\Delta+1\right)
\]
for a universal constant \(C>0\). In particular, \(R_{\mathrm{tail},\infty}^{\rm weak}<\infty\) almost surely.
\end{theorem}

\begin{proof}
Let \(\mathcal H_s\) denote the history after epoch \(s\), and let \(W_s\) be the weak regret incurred during epoch \(s\). Define
\[
    p_s:=\Pr(c_{s-1}\neq a^\star).
\]
The starting incumbent can be arbitrary, so \(p_1\le1\). By Lemma~\ref{lem:app-one-step-repair}, conditionally on the history before epoch \(s\),
\[
    \Pr(c_s\neq a^\star\mid\mathcal H_{s-1})
    \le
    \mathbf 1\{c_{s-1}\neq a^\star\}(\rho_s+\kappa_s)
    +
    \mathbf 1\{c_{s-1}=a^\star\}\kappa_s
    \le \rho_s+\kappa_s.
\]
Taking expectations gives
\[
    p_{s+1}
    =\Pr(c_s\neq a^\star)
    \le \rho_s+
    \kappa_s
    =2\cdot4^{-(s+2)}.
\]
Equivalently, for every \(s\ge2\),
\[
    p_s\le 2\cdot4^{-(s+1)},
\]
while \(p_1\le1\). Thus \(p_s\) is geometrically summable from \(s=2\) onward.

Applying Lemma~\ref{lem:app-screen-epoch-regret} conditionally on the history before epoch \(s\) gives
\[
\begin{split}
\mathbb E[W_s\mid\mathcal H_{s-1}]
\le\frac{CN}{\Delta^2}\Bigg[&
\mathbf 1\{c_{s-1}\neq a^\star\}
\left(\log\frac1{\rho_s}+\log\frac1{\kappa_s}+\log\frac1\Delta+1\right)\\
&+\rho_s\left(\log\frac1{\kappa_s}+\log\frac1\Delta+1\right)\Bigg].
\end{split}
\]
Taking expectations yields the same bound with \(\mathbf 1\{c_{s-1}\neq a^\star\}\) replaced by \(p_s\). Because \(\rho_s=\kappa_s=4^{-(s+2)}\), we have \(\log(1/\rho_s)=\log(1/\kappa_s)=(s+2)\log4\). Using \(p_1\le1\), \(p_s\le2\cdot4^{-(s+1)}\) for \(s\ge2\), and \(\rho_s=4^{-(s+2)}\),
\[
    \sum_{s\ge1}(p_s+\rho_s)(s+1)<\infty,
    \qquad
    \sum_{s\ge1}(p_s+\rho_s)<\infty,
\]
because both are bounded by universal multiples of \(1+\sum_{s\ge1}s4^{-s}\). The first series controls the factors \(\log(1/\rho_s)\) and \(\log(1/\kappa_s)\), while the second controls the common factor \(\log(1/\Delta)+1\). Summing over \(s\) yields
\[
    \sum_{s\ge1}\mathbb E[W_s]
    \le
    \frac{CN}{\Delta^2}
    \left(\log\frac1\Delta+1\right).
\]
Each epoch contains finitely many fresh \textsc{DuelTest} and \textsc{BKT} calls, and each call terminates almost surely under \(\Delta>0\). Hence all epochs terminate almost surely and the infinite comparison stream is well-defined. Since the epochs are consecutive and disjoint in time,
\[
    R_{\mathrm{tail},\infty}^{\rm weak}=\sum_{s\ge1}W_s,
\]
and monotone convergence gives \(\mathbb E[R_{\mathrm{tail},\infty}^{\rm weak}]=\sum_s\mathbb E[W_s]\), proving the expected-regret bound.

Finally, a nonnegative random variable with finite expectation is finite almost surely. Hence \(R_{\mathrm{tail},\infty}^{\rm weak}<\infty\) almost surely.
\end{proof}

\begin{corollary}[Constant-confidence warm start]
\label{cor:app-weak-warm-start}
Under Assumption~\ref{assump:condorcet-gap}, Algorithm~\ref{alg:tlcr-weak} with the constant-confidence warm start in line 1 satisfies
\[
    \mathbb E[R_\infty^{\rm weak}]
    \le
    \frac{CN}{\Delta^2}
    \left(\log\frac1\Delta+1\right)
\]
for a universal constant \(C>0\). In particular, \(R_\infty^{\rm weak}<\infty\) almost surely.
\end{corollary}

\begin{proof}
Let \(W_0\) be the weak regret incurred by the warm start. Since weak regret is at most the number of comparisons, Theorem~\ref{thm:app-plr-cost} with confidence \(1/4\) gives
\[
    \mathbb E[W_0]
    \le
    \frac{CN}{\Delta^2}
    \left(\log\frac1\Delta+1\right).
\]
By Theorem~\ref{thm:app-online-weak}, the tail contributes at most the same order uniformly over the starting incumbent. Since
\[
    R_\infty^{\rm weak}=W_0+R_{\mathrm{tail},\infty}^{\rm weak},
\]
the claim follows after absorbing constants. The almost-sure statement follows because both nonnegative terms on the right-hand side are finite almost surely.
\end{proof}

\begin{corollary}[High-confidence warm start]
\label{cor:app-weak-high-confidence-warm-start}
Under Assumption~\ref{assump:condorcet-gap}, if line 1 of Algorithm~\ref{alg:tlcr-weak} is replaced by \(c_0\leftarrow\textsc{TreeAscent}(b_0,\mathcal T,\delta)\) and \(c_0\) is reported as the BAI answer when this pass finishes, at time \(\tau_{\mathrm{Dual}}\), then \(\Pr(c_0=a^\star)\ge1-\delta\) and
\[
\max\left\{
\mathbb E[\tau_{\mathrm{Dual}}],\;
\mathbb E[R_\infty^{\mathrm{weak}}]
\right\}
\le\frac{CN}{\Delta^2}
\left(\log\frac1\delta+\log\frac1\Delta+1\right).
\]
\end{corollary}

\begin{proof}
The BAI correctness and warm-start comparison bound follow from Theorems~\ref{thm:app-plr-correctness} and~\ref{thm:app-plr-cost}. The weak regret of the warm start is bounded by its number of comparisons, and Theorem~\ref{thm:app-online-weak} adds a tail term of order \(CN\Delta^{-2}(\log(1/\Delta)+1)\), which is absorbed into the displayed bound. The warm start is exactly the \textsc{TreeAscent} call of \textsc{TG-ITE-BAI}, and screening makes no decision before it ends. With the same inputs and feedback, both policies thus report the same arm at the same time.
\end{proof}

\subsection{From gap-weighted to binary regret}
\label{app:weak-lower-scope}
In each round, checking whether zero, one, or both selected arms equal $a^\star$ shows that the gap-weighted losses of Remark~\ref{rem:regret-conventions} satisfy
\begin{equation}
\label{eq:regret-conversion}
 \Delta^\star r_t^{\rm weak}\le r_t^{{\rm weak},{\rm gw}}\le \tfrac12 r_t^{\rm weak},
 \qquad
 \tfrac{\Delta^\star}{2}r_t^{\rm str}\le r_t^{{\rm str},{\rm gw}}\le \tfrac12 r_t^{\rm str}.
\end{equation}
Summing over rounds and taking expectations preserves these relations. By the left inequality for weak regret, the gap-weighted bound of \citet[Theorem~4.2]{saad2024weak} implies that WR-TINF attains binary weak regret $O(N/(\Delta^\star)^2)$ on every Condorcet instance. The next corollary shows that this order is unavoidable in the worst case.

\begin{corollary}[Binary weak-regret lower bound]
\label{cor:weak-binary-lower-saad}
Let $N\ge6$, $d\in(0,1/4)$, and let $\mathcal D(d)$ be the class of $N$-arm instances with a Condorcet winner $a^\star$ satisfying $p_{a^\star i}-1/2=d$ for every $i\ne a^\star$ and $|p_{ij}-1/2|\le d$ for nonwinners $i,j$. There are universal constants $c,c'>0$ such that, for every policy and $T\ge c'N/d^2$,
\[
\sup_{P\in\mathcal D(d)}\mathbb E_P[R_T^{\mathrm{weak}}]
\ge cN/d^2.
\]
\end{corollary}

\begin{proof}
All winner gaps in $\mathcal D(d)$ equal $d$, so in every round the gap-weighted weak loss is exactly $d$ times the binary one, including self-comparisons. \citet[Theorem~5.1]{saad2024weak} bound the supremum of the expected gap-weighted weak regret over $\mathcal D(d)$ below by $cN/d$ in this horizon regime. Dividing by $d$ proves the claim.
\end{proof}

\section{Strong Regret}
\label{app:strong-proofs}

This appendix proves the strong-regret guarantees. We use the self-play convention from the problem setting: a self-comparison $(i,i)$ produces dummy feedback and has zero strong regret only when $i=a^\star$. Every non-dummy comparison made by \textsc{TreeAscent} is between two distinct arms, and therefore contributes at most one unit of binary strong regret.

\subsection{\textsc{TG-ITE-Strong-KnownHorizon}}
\label{app:strong-known-horizon}

\begin{algorithm}[t]
\caption{\textsc{TG-ITE-Strong-KnownHorizon}\((b_0,\mathcal T, H)\)}
\label{alg:tlcr-strong-known}
\begin{algorithmic}[1]
\REQUIRE Starting arm \(b_0\); fixed balanced tree \(\mathcal T\); horizon \(H\ge1\).
\STATE Set \(\epsilon_H\leftarrow 1/(H+1)\).
\STATE Run \textsc{TreeAscent}\((b_0,\mathcal T,\epsilon_H)\) until the pass terminates or the horizon \(H\) expires.
\IF{the pass terminates by time \(H\) with output \(\widehat a\)}
    \STATE For every remaining round up to time \(H\), play the self-comparison \((\widehat a,\widehat a)\).
\ELSE
    \STATE Halt at time \(H\); all rounds up to \(H\) were spent inside the identification pass.
\ENDIF
\end{algorithmic}
\end{algorithm}

\begin{theorem}[Known-horizon strong regret]
\label{thm:app-strong-known}
Under Assumption~\ref{assump:condorcet-gap}, for every horizon $H\ge1$, Algorithm~\ref{alg:tlcr-strong-known} satisfies
\[
    \mathbb E[R_H^{\rm str}]
    \le
    \frac{CN}{\Delta^2}
    \left(
        \log(H+1)+
        \log\frac1\Delta+1
    \right)+1.
\]
\end{theorem}

\begin{proof}
Let $\tau_H$ and $\widehat a$ be the stopping time and eventual output of the full \textsc{TreeAscent} pass if it were allowed to run until termination. This is a counterfactual definition relative to the finite horizon: when $\tau_H>H$, the actual algorithm never computes or uses $\widehat a$ before time $H$. The pair $(\tau_H,\widehat a)$ is nevertheless well-defined because the pass terminates almost surely. During the identification pass, every realized comparison is a distinct-arm comparison and hence contributes at most one unit of binary strong regret. In the event $\tau_H>H$, the algorithm spends all $H$ rounds inside the pass and then halts, so $\widehat a$ has no effect on the realized regret. If $\tau_H\le H$, the remaining self-play rounds contribute strong regret only on the event $\{\widehat a\neq a^\star\}$. Pathwise,
\[
    R_H^{\rm str}
    \le
    \min\{\tau_H,H\}
    +(H-\tau_H)_+\mathbf 1\{\widehat a\neq a^\star\}
    \le
    \tau_H+H\mathbf 1\{\widehat a\neq a^\star\}.
\]
Taking expectations and applying Theorems~\ref{thm:app-plr-correctness} and~\ref{thm:app-plr-cost} with $\varepsilon_H=1/(H+1)$ gives
\[
    \mathbb E[R_H^{\rm str}]
    \le
    \frac{CN}{\Delta^2}
    \left(\log\frac1{\varepsilon_H}+\log\frac1\Delta+1\right)
    +H\varepsilon_H.
\]
Since $\log(1/\varepsilon_H)=\log(H+1)$ and $H\varepsilon_H\le1$, the theorem follows.
\end{proof}

\subsection{\textsc{TG-ITE-Strong-Anytime}}
\label{app:strong-horizonless}

\begin{theorem}[Horizonless strong regret]
\label{thm:app-strong-anytime}
Under Assumption~\ref{assump:condorcet-gap}, for every horizon $T\ge2$, Algorithm~\ref{alg:tlcr-strong-anytime} satisfies
\[
    \mathbb E[R_T^{\rm str}]
    \le
    \frac{CN}{\Delta^2}
    \left[
        \log(T+1)
        +
        (\log\log(T+2)+1)
        \left(\log\frac1\Delta+1\right)
    \right]
    +C'(\log\log(T+2)+1)
\]
for universal constants $C,C'>0$. For every fixed $\Delta>0$, the leading term as $T\to\infty$ is $O(N\Delta^{-2}\log T)$. The additional $O(N\Delta^{-2}\log\log T\,\log(1/\Delta))$ term is explicitly retained and may dominate when the instance gap is very small.
\end{theorem}

The $C'(\log\log(T+2)+1)$ term below comes from the $S_T$ many self-play error costs $B_s\varepsilon_s=1$, and is independent of $N$ and $\Delta$.

\begin{proof}
We first justify the phase variables used below. Conditional on any history before a phase starts, the next \textsc{TreeAscent} call is a finite composition of fresh pairwise tests and tournaments, and therefore terminates almost surely under Assumption~\ref{assump:condorcet-gap}. By induction over phases, for every fixed $s$, phase $s$ starts and terminates almost surely. Hence the full identification cost $\tau_s$ and the output $\widehat a_s$ of phase $s$ are well-defined. Here $\tau_s$ counts only the distinct-arm comparison rounds used by that identification pass. For a phase that has started but not finished by time $T$, $\widehat a_s$ is counterfactual only relative to the truncated horizon, and using it below only produces an upper bound.

Let $K_T$ be the number of phases that have started by time $T$. If phase $s$ starts at time $t_s$, then
\[
    t_s
    =
    1+\sum_{r=1}^{s-1}(\tau_r+B_r)
    \ge
    \sum_{r=1}^{s-1}B_r .
\]
Thus any phase $s$ that has started by time $T$ satisfies $\sum_{r=1}^{s-1}B_r\le T$. For $s\ge2$, the left side is at least $B_{s-1}=2^{2^{s-1}}$. Hence any such started phase satisfies $2^{2^{s-1}}\le T$, and, after harmless endpoint slack handled by $T+2$,
\[
    s\le 2+\left\lceil\log_2\log_2(T+2)\right\rceil=:S_T .
\]
Therefore $K_T\le S_T$ deterministically. Identification passes can only delay later phases and cannot increase this bound.

Strong regret before time $T$ is upper-bounded pathwise by the full identification costs and full self-play-block losses of all phases that have started:
\[
    R_T^{\rm str}
    \le
    \sum_{s=1}^{K_T}\tau_s
    +
    \sum_{s=1}^{K_T}B_s\mathbf 1\{\widehat a_s\neq a^\star\}.
\]
This may count parts of a pass or a self-play block beyond time $T$, but only as a nonnegative overcount. Since $\tau_s\ge0$ and $B_s\mathbf 1\{\widehat a_s\neq a^\star\}\ge0$, the deterministic bound $K_T\le S_T$ gives
\[
    R_T^{\rm str}
    \le
    \sum_{s=1}^{S_T}\tau_s
    +
    \sum_{s=1}^{S_T}B_s\mathbf 1\{\widehat a_s\neq a^\star\}.
\]
Because $S_T$ is deterministic, we can exchange the finite summation and expectation after taking expectations.

Let $\mathcal F_{s-1}$ denote the history before phase $s$ starts, including the starting arm $c_{s-1}$ chosen by the algorithm. Conditional on $\mathcal F_{s-1}$, this starting arm is fixed, and the $s$-th identification pass uses fresh samples. Since the \textsc{TreeAscent} guarantees are uniform over the starting arm, they apply conditionally:
\[
    \Pr(\widehat a_s\neq a^\star\mid\mathcal F_{s-1})\le\varepsilon_s,
    \qquad
    \mathbb E[\tau_s\mid\mathcal F_{s-1}]
    \le
    \frac{CN}{\Delta^2}
    \left(\log\frac1{\varepsilon_s}+\log\frac1\Delta+1\right).
\]
Taking expectations and summing over $s\le S_T$ yields
\[
    \mathbb E[R_T^{\rm str}]
    \le
    \frac{CN}{\Delta^2}
    \sum_{s=1}^{S_T}
    \left(\log\frac1{\varepsilon_s}+\log\frac1\Delta+1\right)
    +
    \sum_{s=1}^{S_T}B_s\varepsilon_s.
\]
By construction $B_s\varepsilon_s=1$ and $\log(1/\varepsilon_s)=2^s\log2$, so
\[
    \mathbb E[R_T^{\rm str}]
    \le
    \frac{CN}{\Delta^2}
    \left(
        \sum_{s=1}^{S_T}2^s
        +S_T\left(\log\frac1\Delta+1\right)
    \right)
    +S_T.
\]
Finally, the definition of $S_T$ gives $S_T\le3+\log_2\log_2(T+2)$ and therefore
\[
    \sum_{s=1}^{S_T}2^s
    =2^{S_T+1}-2
    \le
    16\log_2(T+2)
    \le
    C\log(T+1),
    \qquad
    S_T\le C(\log\log(T+2)+1).
\]
Substituting these estimates proves the theorem. The additive $C'(\log\log(T+2)+1)$ term is the absorbed form of $S_T=\sum_{s=1}^{S_T}B_s\varepsilon_s$. The induction argument at the beginning of the proof also shows that the anytime policy is well-defined as an infinite comparison stream.
\end{proof}

\subsection{A lower bound for strong regret}
\label{app:strong-lower-scope}

\begin{proposition}[Homogeneous-gap strong-regret lower bound]
\label{prop:strong-binary-homogeneous-lower}
Fix $N\ge2$ and $d\in(0,1/2)$, and let $P$ be the homogeneous total-order instance with $p_{i,j}=1/2+d$ for $i<j$, so that $\Delta=\Delta^\star=d$. Consider a policy that is strongly consistent over strict total orders, that is, $\mathbb E_Q[R_T^{\mathrm{str}}]=o(T^a)$ for every strict-total-order instance $Q$ and every $a>0$. Then
\[
\liminf_{T\to\infty}\frac{\mathbb E_P[R_T^{\mathrm{str}}]}{\log T}
\ge
\frac{N-1}{\operatorname{kl}(1/2-d,1/2)}
\ge\frac{N-1}{4d^2}.
\]
\end{proposition}

\begin{proof}
By~\eqref{eq:regret-conversion}, strong consistency for binary regret implies strong consistency for gap-weighted regret. Let $N_{ij}(T)$ be the number of comparisons of the pair $\{i,j\}$, $i\ne j$, up to time $T$. Each arm $i>1$ is beaten exactly by the arms $1,\ldots,i-1$, and Lemma~1(ii) of \citet{komiyama2015regret}, applied to the class of strict total orders, gives
\[
\operatorname{kl}(1/2-d,1/2)
\sum_{j<i}\mathbb E_P[N_{ij}(T)]
\ge(1-o(1))\log T.
\]
Summing over $i=2,\ldots,N$ counts each pair once, and every comparison of two distinct arms incurs one unit of binary strong regret. Hence
\[
\mathbb E_P[R_T^{\mathrm{str}}]
\ge\sum_{i=2}^{N}\sum_{j<i}\mathbb E_P[N_{ij}(T)]
\ge\frac{(N-1)(1-o(1))\log T}{\operatorname{kl}(1/2-d,1/2)}.
\]
Dividing by $\log T$ and letting $T\to\infty$ proves the first bound. For the second, $\log x\le x-1$ gives $\operatorname{kl}(p,q)\le(p-q)^2/(q(1-q))$, so $\operatorname{kl}(1/2-d,1/2)\le4d^2$.
\end{proof}

Every strict total order satisfies Assumption~\ref{assump:condorcet-gap}, so by Theorem~\ref{thm:app-strong-anytime} the anytime policy is strongly consistent over this class. On $P$, its bound grows as $O(Nd^{-2}\log T)$ and therefore matches the lower bound up to constants as $T\to\infty$.

\section{Known-Gap Budgeted Extension}
\label{app:known-gap-budgeted-proofs}

This appendix proves the guarantees stated in Section~\ref{sec:known-gap-budgeted-extension}. Throughout this section, \(g\le\Delta^\star\) is the known lower bound, ties among nonwinners are allowed, and all calls use fresh samples. The proofs mirror the gap-oblivious analysis, but the pairwise cost is now a deterministic fixed budget.

\begin{algorithm}[H]
\caption{\textsc{BudgetedDuelTest}\((u,v,g,\delta)\)}
\label{alg:budgeted-dueltest}
\begin{algorithmic}[1]
\REQUIRE Arms \(u\neq v\); budget parameter \(g\in(0,1/2]\); confidence level \(\delta\in(0,1)\).
\STATE Set \(M\leftarrow \left\lceil \frac{1}{2g^2}\log\frac{4}{\delta}\right\rceil\) and \(S\leftarrow0\).
\FOR{\(n=1,2,\ldots,M\)}
    \STATE Observe \(Y_n\sim\mathrm{Bernoulli}(p_{u,v})\), where \(Y_n=1\) means that \(u\) wins.
    \STATE \(S\leftarrow S+Y_n\), \(\widehat\mu_n\leftarrow S/n\), and \(b_n(\delta)\leftarrow \sqrt{\frac{\log(8n^2/\delta)}{2n}}\).
    \IF{\(|\widehat\mu_n-1/2|>b_n(\delta)-g\)}
        \IF{\(\widehat\mu_n\ge1/2\)}
            \STATE \textbf{return} \(u\).
        \ELSE
            \STATE \textbf{return} \(v\).
        \ENDIF
    \ENDIF
\ENDFOR
\IF{\(\widehat\mu_M\ge1/2\)}
    \STATE \textbf{return} \(u\).
\ELSE
    \STATE \textbf{return} \(v\).
\ENDIF
\end{algorithmic}
\end{algorithm}

\begin{proof}[Proof of Lemma~\ref{lem:budgeted-dueltest-main}]
The deterministic budget is immediate from the loop length. We prove correctness for the case \(p_{u,v}=1/2+\gamma\) with \(\gamma\ge g\). Let \(M=\lceil (2g^2)^{-1}\log(4/\delta)\rceil\) and \(b_n=b_n(\delta)\).

Consider first an erroneous early return. Since ties are broken in favor of \(u\), the algorithm can return \(v\) only if, for some \(n\le M\), \(\widehat\mu_n<1/2\) and \(|\widehat\mu_n-1/2|>b_n-g\). If \(b_n\ge g\), then
\[
    p_{u,v}-\widehat\mu_n
    =\gamma+\left(\frac12-\widehat\mu_n\right)
    > g+b_n-g=b_n.
\]
If \(b_n<g\), then \(p_{u,v}-\widehat\mu_n>g>b_n\). Hence every erroneous early return is contained in
\[
    \left\{\exists n\le M: p_{u,v}-\widehat\mu_n>b_n\right\}.
\]
A one-sided Hoeffding bound and a union bound give
\[
    \Pr\left(\exists n\le M: p_{u,v}-\widehat\mu_n>b_n\right)
    \le
    \sum_{n\ge1}\exp(-2nb_n^2)
    =
    \sum_{n\ge1}\frac{\delta}{8n^2}
    \le \frac\delta4.
\]
If no early return occurs and the final majority vote is wrong, then \(\widehat\mu_M<1/2\), and therefore \(p_{u,v}-\widehat\mu_M>\gamma\ge g\). By Hoeffding's inequality,
\[
    \Pr(\widehat\mu_M<1/2)
    \le
    \exp(-2Mg^2)
    \le \frac\delta4.
\]
Thus the error probability is at most \(\delta/2\), and in particular at most \(\delta\). The proof for \(p_{u,v}=1/2-\gamma\) is identical with the inequalities reversed, except that an empirical tie now returns the wrong arm. Hoeffding's inequality bounds these non-strict events in the same way.
\end{proof}

For the remaining proofs, let \(\textsc{BKT}_g\), \(\textsc{TreeAscent}_g\), and \(\textsc{ScreenAndReplace}_g\) denote the natural variants obtained by using \textsc{BudgetedDuelTest} instead of \textsc{DuelTest} everywhere.

\begin{lemma}[Budgeted knockout tournament]
\label{lem:app-budgeted-bkt}
Let \(S\subseteq[N]\) be nonempty and \(m=|S|\). For every \(\delta\in(0,1)\), \(\textsc{BKT}_g(S,\delta)\) uses at most
\[
    \frac{Cm}{g^2}\left(\log\frac1\delta+1\right)
\]
comparisons deterministically for a universal constant \(C>0\). Moreover, if \(S\) contains an arm \(w\) such that \(p_{w,j}\ge1/2+g\) for every \(j\in S\setminus\{w\}\), then \(\textsc{BKT}_g(S,\delta)\) returns \(w\) with probability at least \(1-\delta\).
\end{lemma}

\begin{proof}
The correctness proof only tracks \(w\). In tournament round \(r\), \(w\) participates in at most one pairwise call with confidence \(\delta_r=6\delta/(\pi^2r^2)\). Whenever \(w\) is one of the two inputs, the pairwise gap is at least \(g\), so Lemma~\ref{lem:budgeted-dueltest-main} bounds the conditional probability of eliminating \(w\) by \(\delta_r\). A union bound gives \(\sum_{r\ge1}\delta_r=\delta\).

For the deterministic budget, let \(K_r\) be the number of matches in round \(r\). The halving structure gives \(\sum_r K_r\le m-1\) and \(\sum_r K_r\log r\le Cm\), exactly as in the proof of Lemma~\ref{lem:app-bkt-subset}. Each round-\(r\) call uses at most
\[
    \left\lceil\frac{1}{2g^2}\log\frac4{\delta_r}\right\rceil
    \le
    \frac{C}{g^2}\left(\log\frac1\delta+\log r+1\right)
\]
comparisons. Summing over rounds and using the two halving bounds proves the budget bound.
\end{proof}

\begin{lemma}[Budgeted tree ascent]
\label{lem:app-budgeted-tree-ascent}
If \(g\le\Delta^\star\), then for every starting arm \(b\), tree \(\mathcal T\), and confidence \(\varepsilon\in(0,1)\),
\[
    \Pr\{\textsc{TreeAscent}_g(b,\mathcal T,\varepsilon)=a^\star\}\ge1-\varepsilon,
\]
and the number of comparisons in the pass is deterministically bounded by
\[
    \tau_{{\rm TA},g}(b,\varepsilon)
    \le
    \frac{CN}{g^2}\left(\log\frac1\varepsilon+1\right).
\]
The same deterministic bound upper-bounds the weak or strong regret incurred by the pass.
\end{lemma}

\begin{proof}
The capture-and-preservation proof is the same as Theorem~\ref{thm:app-plr-correctness}. If \(a^\star\) first enters the ancestor path through the critical block \(Q_{h_\star}(b)\), then within this block it beats every other arm by at least \(g\). Lemma~\ref{lem:app-budgeted-bkt} therefore returns \(a^\star\) from the critical block with failure probability at most \(\beta_{h_\star}\). Once \(a^\star\) is the current champion, every later merge call involving \(a^\star\) has pairwise gap at least \(g\), so Lemma~\ref{lem:budgeted-dueltest-main} bounds its failure probability by \(\eta_h\). Thus the same union bound gives \(\sum_h\beta_h+\sum_h\eta_h\le\varepsilon\).

For cost, the subtree tournaments satisfy, by Lemma~\ref{lem:app-budgeted-bkt},
\[
    \sum_{h=1}^L T_h^{\rm sub}
    \le
    \frac{C}{g^2}\left(\sum_{h=1}^L m_h\log\frac1{\beta_h}+N\right)
    \le
    \frac{CN}{g^2}\left(\log\frac1\varepsilon+1\right),
\]
where the final inequality is the same size-weighted entropy calculation as in Theorem~\ref{thm:app-plr-cost}. The merge calls contribute at most
\[
    \frac{CL}{g^2}\left(\log\frac{2L}{\varepsilon}+1\right)
    \le
    \frac{CN}{g^2}\left(\log\frac1\varepsilon+1\right),
\]
because \(L\le M<2N\) and \(L\log(2L)\le CN\). Adding the two displays proves the deterministic comparison budget. Each comparison costs at most one unit of weak or strong regret, so the same bound applies to regret incurred during identification.
\end{proof}

\begin{lemma}[Budgeted screen-and-replace epoch]
\label{lem:app-budgeted-screen-epoch}
Let \(\mathcal F\) be the history before a call to \(\textsc{ScreenAndReplace}_g(c,\rho,\kappa)\), and suppose that \(c\) is \(\mathcal F\)-measurable. Let \(c^+\) be the returned arm and let \(W(c,\rho,\kappa)\) be the weak regret incurred during the call. If \(g\le\Delta^\star\), then, almost surely,
\[
    \Pr(c^+\neq a^\star\mid\mathcal F)
    \le
    \mathbf 1\{c\neq a^\star\}(\rho+\kappa)+\mathbf 1\{c=a^\star\}\kappa,
\]
and
\[
    \mathbb E[W(c,\rho,\kappa)\mid\mathcal F]
    \le
    \frac{CN}{g^2}\left[
        \mathbf 1\{c\neq a^\star\}\left(\log\frac1\rho+\log\frac1\kappa+1\right)
        +
        \rho\left(\log\frac1\kappa+1\right)
    \right].
\]
\end{lemma}

\begin{proof}
The correction probability is identical to Lemma~\ref{lem:app-one-step-repair}. If \(c\neq a^\star\), then the screening phase compares \(c\) with \(a^\star\), and this comparison inserts \(a^\star\) into the beater set with failure probability at most \(\rho\). Conditional on \(a^\star\) being present, the replacement tournament contains an arm that beats every other candidate by at least \(g\), so Lemma~\ref{lem:app-budgeted-bkt} returns \(a^\star\) with failure probability at most \(\kappa\). If \(c=a^\star\), only the replacement tournament can lose \(a^\star\), again with probability at most \(\kappa\).

For the regret bound, the screening part costs at most \((N-1)\lceil(2g^2)^{-1}\log(4/\rho)\rceil\) comparisons, but contributes weak regret only when \(c\neq a^\star\). If \(c\neq a^\star\), the replacement tournament has size at most \(N\) and costs at most \(CNg^{-2}(\log(1/\kappa)+1)\) by Lemma~\ref{lem:app-budgeted-bkt}. If \(c=a^\star\), a nonwinner enters the beater set only after a false positive against \(a^\star\), which has conditional probability at most \(\rho\). Hence \(\mathbb E[|B|\mid\mathcal F,c=a^\star]\le(N-1)\rho\). Applying the deterministic tournament budget conditionally on \(B\), and using \((|B|+1)\mathbf 1\{|B|>0\}\le2|B|\), yields the false-positive term \(CN\rho g^{-2}(\log(1/\kappa)+1)\). Summing the contributions proves the display.
\end{proof}

\begin{proof}[Proof of Theorem~\ref{thm:known-gap-extension-main}]
The BAI statement is Lemma~\ref{lem:app-budgeted-tree-ascent} with \(\varepsilon=\delta\). For the constant-confidence weak-regret algorithm, the warm start costs at most \(CN/g^2\) regret by Lemma~\ref{lem:app-budgeted-tree-ascent}. For the tail, use Lemma~\ref{lem:app-budgeted-screen-epoch} with \(\rho_s=\kappa_s=4^{-(s+2)}\). The same recursion as in Theorem~\ref{thm:app-online-weak} gives \(p_1\le1\) and \(p_s\le2\cdot4^{-(s+1)}\) for \(s\ge2\). Since \(\log(1/\rho_s)=\log(1/\kappa_s)=(s+2)\log4\), taking expectations in Lemma~\ref{lem:app-budgeted-screen-epoch} bounds the weak regret of epoch \(s\) by \(CNg^{-2}(p_s+\rho_s)(s+1)\). The series \(\sum_{s\ge1}(p_s+\rho_s)(s+1)\) is finite, so summing the epoch bounds gives \(\mathbb E[R_\infty^{\rm weak}]\le CN/g^2\). With a \(\delta\)-confidence warm start, the warm-start contribution becomes \(CN g^{-2}(\log(1/\delta)+1)\), and the tail is absorbed.

For known-horizon strong regret, the pathwise inequality
\[
    R_H^{\rm str}\le \tau_H+H\mathbf 1\{\widehat a\neq a^\star\}
\]
from Theorem~\ref{thm:app-strong-known} is unchanged. Applying Lemma~\ref{lem:app-budgeted-tree-ascent} with \(\varepsilon_H=1/(H+1)\) bounds \(\tau_H\) by \(CNg^{-2}(\log(H+1)+1)\) and \(H\Pr(\widehat a\neq a^\star)\) by \(H\varepsilon_H\le1\), which gives the stated bound.

For the anytime strong-regret policy, the phase-counting argument of Theorem~\ref{thm:app-strong-anytime} is unchanged: by time \(T\), at most \(S_T=O(\log\log(T+2)+1)\) phases have started, and \(\sum_{s\le S_T}\log(1/\varepsilon_s)=O(\log(T+1))\). Lemma~\ref{lem:app-budgeted-tree-ascent} replaces each phase cost by \(CN g^{-2}(\log(1/\varepsilon_s)+1)\), while the self-play error term remains \(\sum_{s\le S_T}B_s\varepsilon_s=S_T\). Therefore
\[
    \mathbb E[R_T^{\rm str}]
    \le
    \frac{CN}{g^2}\left(\log(T+1)+\log\log(T+2)+1\right)
    +C'(\log\log(T+2)+1),
\]
which completes the proof.
\end{proof}

\section{Nonwinner-Gap Dependence of the Weak-Regret Policy}
\label{app:gap-obstruction}
We prove Proposition~\ref{prop:global-gap-obstruction} in two steps. We first show that \textsc{DuelTest} at confidence $1/64$, the screening confidence of the first \textsc{ScreenAndReplace} call, needs in expectation at least order $\gamma^{-2}\log(1/\gamma)$ comparisons on a pair with gap $\gamma$. We then build a four-arm instance on which the constant-confidence warm start returns a nonwinner $u$ with probability independent of $\gamma$. The first screen then runs this test between $u$ and a nonwinner $v$ at gap $\gamma$, and each comparison costs one unit of weak regret.

\subsection{Duration of \textsc{DuelTest} near a tie}
\label{app:dueltest-lower-duration}
\begin{lemma}[Near-tie duration of \textsc{DuelTest}]
\label{lem:near-tie-duration}
Let $\delta_0=1/64$ and let $\tau_\gamma$ be the stopping time of Algorithm~\ref{alg:dueltest} with confidence $\delta_0$ on i.i.d. Bernoulli$(1/2+\gamma)$ observations. For every sufficiently small $\gamma>0$, put
\[
    m_\gamma=
    \left\lfloor\frac{\gamma^{-2}\log(1/\gamma)}{100}\right\rfloor.
\]
Then $\Pr(\tau_\gamma>m_\gamma)\ge63/64$, and hence
$\mathbb E[\tau_\gamma]=\Omega(\gamma^{-2}\log(1/\gamma))$.
\end{lemma}

\begin{proof}
Write $c_n=c(n,\delta_0)$. For real $n\ge1$, the derivative of $c_n^2$ has the sign of
\[
    1+\frac{n}{n+1}-\log\frac{2n(n+1)}{\delta_0}
    <2-\log256<0,
\]
so $c_n$ decreases in $n$. Since $2n(n+1)/\delta_0\ge n^2$, we also have $c_n^2\ge\log n/n$. For sufficiently small $\gamma$, $m_\gamma\ge\gamma^{-1}$. Using this bound in the numerator and the definition of $m_\gamma$ in the denominator gives
\[
    c_{m_\gamma}^2
    \ge \frac{\log m_\gamma}{m_\gamma}
    \ge\frac{\log(1/\gamma)}{\gamma^{-2}\log(1/\gamma)/100}
      =100\gamma^2.
\]
Since $c_n$ decreases, $\gamma\le c_n/10$ for every $n\le m_\gamma$. If the test stops at such an $n$, then
\[
    |\widehat\mu_n-(1/2+\gamma)|
    \ge|\widehat\mu_n-1/2|-\gamma
    > c_n-\gamma\ge\tfrac9{10}c_n,
\]
and Hoeffding's inequality bounds the probability of this event by
\[
    2\exp\{-2n(9c_n/10)^2\}
    =2\left(\frac{\delta_0}{2n(n+1)}\right)^{81/100}.
\]
For $n\le16$, stopping is impossible on every observation sequence, because $c_n\ge c_{16}>1/2$ while $|\widehat\mu_n-1/2|\le1/2$. For larger $n$, a union bound, the bound $n(n+1)\ge n^2$, and the integral test give
\begin{align*}
    \Pr(\tau_\gamma\le m_\gamma)
    &\le 2\sum_{n=17}^{m_\gamma}
       \left(\frac{\delta_0}{2n(n+1)}\right)^{81/100}\\
    &\le 2\left(\frac1{128}\right)^{81/100}
          \int_{16}^{\infty}x^{-162/100}\,dx\\
    &=\frac{100}{31}\,2^{-163/20}
      <4\cdot2^{-8}=\frac1{64}.
\end{align*}
The expected-duration bound follows from $\mathbb E[\tau_\gamma]\ge m_\gamma\Pr(\tau_\gamma>m_\gamma)$ and $m_\gamma\ge\gamma^{-2}\log(1/\gamma)/200$ for sufficiently small $\gamma$.
\end{proof}

\subsection{A wrong warm start forces a near-tied comparison}
\label{app:weak-gap-obstruction}
\begin{proof}[Proof of Proposition~\ref{prop:global-gap-obstruction}]
Use four arms $a^\star,u,v,w$ with
\[
    p_{a^\star,x}=\tfrac12+d\quad(x\in\{u,v,w\}),\qquad
    p_{u,v}=\tfrac12+\gamma,\qquad
    p_{u,w}=p_{v,w}=\tfrac12+d,
\]
together with $p_{j,i}=1-p_{i,j}$ and $p_{i,i}=1/2$. Then $\Delta^\star=d$ and $\Delta=\gamma$ for $0<\gamma<d$. Arrange the leaves as $(u,w,a^\star,v)$ and start from $u$. The constant-confidence warm start runs \textsc{TreeAscent} with $\varepsilon=1/4$ and $L=2$, so each merge has confidence $\eta_h=1/16$. The pass first merges $u$ with $w$. It then runs \textsc{BKT} on $\{a^\star,v\}$ with confidence $\beta_2=1/12$, whose single match has confidence $6\beta_2/\pi^2=1/(2\pi^2)$, and finally merges the two winners.

For $\eta\in(0,1)$, let $k(\eta)=\min\{n\ge1:c(n,\eta)<1/2\}$. This is the stopping time of \textsc{DuelTest} at confidence $\eta$ when every observation favors the same arm. Let $E$ be the event that the first merge sees $k(1/16)$ wins of $u$ over $w$, the \textsc{BKT} match sees $k(1/(2\pi^2))$ wins of $a^\star$ over $v$, and the final merge sees $k(1/16)$ wins of $u$ over $a^\star$. On $E$, the warm start returns $u$. These three tests use fresh samples on pairs with gap $d$, so
\[
    q_d:=\Pr(E)
    =(1/2+d)^{k(1/16)+k(1/(2\pi^2))}
      (1/2-d)^{k(1/16)}>0,
\]
which does not depend on $\gamma$.

The first \textsc{ScreenAndReplace} call starts from the incumbent $u$. Its screening confidence is $\rho_1=4^{-3}=\delta_0$, and its incumbent stays $u$ until every challenger has been screened. Hence \textsc{DuelTest}$(u,v,\delta_0)$ is called whatever the earlier screening outcomes are. The earlier screening tests compare $u$ with $a^\star$ or $w$. They have gap $d$ and terminate almost surely, so the $u$--$v$ test is reached. Given their history and $E$, the $u$--$v$ comparisons are fresh Bernoulli$(1/2+\gamma)$ observations, and each incurs one unit of binary weak regret because neither arm is $a^\star$. Lemma~\ref{lem:near-tie-duration} therefore gives
\[
    \mathbb E[R_\infty^{\rm weak}]
    \ge q_d\,\mathbb E[\tau_\gamma]
    \ge \frac{63q_d}{64}\,m_\gamma
    \ge c_d\,\gamma^{-2}\log(1/\gamma)
\]
for a constant $c_d>0$ depending only on $d$. Since $N=4$ and $\Delta^\star=d$ stay fixed while this bound diverges as $\gamma\to0$, Algorithm~\ref{alg:tlcr-weak} has no bound depending only on $N$ and $\Delta^\star$.
\end{proof}

\section{Experimental Protocols and Additional Results}
\label{app:time}

\subsection{Setup}
\label{app:experiment-instances}

We report cumulative binary weak and strong regret as defined in Section~\ref{sec:pre}. Every identification, exploration, and self-play round counts toward the horizon, and regret averages include runs that are still identifying at the horizon. At the horizon, each online identification is completed-correct (C), completed-wrong (W), or unfinished (U). Unless stated otherwise, $\pm$ denotes the sample standard deviation over seeds. Intervals for a mean are 95\% Student-$t$ intervals, and intervals for an event rate are 95\% Wilson intervals. Seeds are paired across algorithms within a configuration.

All matrices satisfy $p_{ii}=1/2$ and $p_{ji}=1-p_{ij}$. Unless stated otherwise, the winner is arm $0$, the starting arm is $N-1$, and the balanced tree follows label order.

\paragraph{Cyclic Condorcet family.}
Set $p_{0j}=1/2+\Delta^\star$ for $j>0$. For nonwinners $1\le i<j\le m$ with $m=N-1$, set $p_{ij}=1/2+d$ if $(j-i)\bmod m\le m/2$ and $p_{ij}=1/2-d$ otherwise. Since $N$ is a power of two, $m$ is odd and no antipodal tie arises. The nonwinners therefore form cycles, and $\Delta=\min\{\Delta^\star,d\}$.

\paragraph{Perturbed cyclic family.}
The main weak-regret instance uses this construction with $\Delta^\star=0.26$ and $d=0.40$. It then adds independent $\mathrm{Uniform}[-0.025,0.025]$ noise to each nonwinner probability $p_{ij}$ with $i<j$ and restores reciprocity. With matrix seed $123$, the nonwinner gaps lie in $[0.375,0.425]$.

\paragraph{Equal-gap total order.}
The joint experiment also uses $p_{ij}=1/2+0.30$ for $i<j$, so arm $0$ is the winner and $\Delta=\Delta^\star=0.30$. In the main online panels and the anytime-tail study, the cyclic winner is relabeled to $128$, with start $255$.

\begin{table}[t]
\centering
\caption{Experiment configurations. Seeds are per algorithm and cell, and $d$ denotes the nonwinner gap of the cyclic family.}
\label{tab:experiment-protocol}
\begin{tabular}{@{}p{0.23\linewidth}p{0.60\linewidth}r@{}}
\toprule
Experiment & Configuration & Seeds \\
\midrule
Main BAI & $N=8,16,32,64,128,256$; $(\Delta^\star,d)=(.32,.38)$; $\delta=.1$ & 100 \\
Main weak & $N=256$; perturbed cyclic; $T=50{,}000$ and $200{,}000$ & 30 \\
Main strong & $N=256$; $(\Delta^\star,d)=(.30,.38)$; $T=100{,}000$ & 30 \\
Gap sweeps & $N=256$; $\Delta^\star=.30$; $d=.38,.20,.10,.05,.025$; online $T=10^6$ & 100 \\
Small winner gap & $N=256$; $(\Delta^\star,d)=(.05,.20)$; $T=10^6$ & 30 \\
Fresh/cache & $N=256$; $(\Delta^\star,d)=(.30,.38)$; $T=10^6$ & 100 pairs \\
Horizon scaling & $N=256$; $T=10^4,3\!\cdot\!10^4,10^5,3\!\cdot\!10^5,10^6$ & 100 \\
Arm scaling & $N=64,128,256,512,1024$; $T=10^6$; $(\Delta^\star,d)=(.30,.38)$ & 10 \\
Additional BAI scaling & $N=8,16,32,64,128,256$; $\Delta^\star=.30$; $d=.38,.05$; $\delta=.1$ & 30 \\
Joint objective & $N=32,64,128,256$; cyclic and total order; $T=50{,}000$; $\delta=.1$ & 100 \\
Anytime tail & $N=256$; $(\Delta^\star,d)=(.30,.38)$; $T=10^5,10^6$ & 100 \\
Failure stress & $N=32$; $\Delta=\Delta^\star=g=.05$; $\delta=.5$; $T=10^6$ & 5,000 \\
\bottomrule
\end{tabular}
\end{table}

Figure~\ref{fig:experiments}(a) uses the first 30 of the 100 BAI seeds, and the native WS-W comparator in the joint study uses the first 30 seeds.

\subsection{Algorithms, side information, and source assumptions}
\label{app:experiment-baselines}

WS-W, WR-TINF, WR-EXP3-IX, and Versatile-DB are adapted from the supplementary code of \citet{saad2024weak}, available at \url{https://openreview.net/forum?id=dY4YGqvfgW}. We corrected the last three to follow their source algorithms, as detailed below, and implemented the rest ourselves.

\paragraph{TG-ITE and BKT.}
Every gap-oblivious test uses the boundary $c(n,\delta)$ and the strict stopping rule of Algorithm~\ref{alg:dueltest}, computed in the log domain to handle very small phase confidences. BKT is Algorithm~\ref{alg:bkt} on all arms. The tree allocation, byes, warm-start confidence $1/4$, screening schedule, known-horizon confidence $1/(T+1)$, and anytime schedule $\epsilon_s=2^{-2^s}$, $B_s=2^{2^s}$ follow the stated algorithms. The known-gap variants use Algorithm~\ref{alg:budgeted-dueltest} with the oracle input $g=\Delta^\star$.

\paragraph{SeqElim and Knockout.}
SeqElim implements Algorithm~1 of \citet{falahatgar2017few}, and Knockout implements the tournament and comparison primitives of \citet{falahatgar2017maximum}. Both target approximate maximum selection. We give them $\delta=0.1$ and $\epsilon=0.95\Delta^\star=0.304$, which uses knowledge of $\Delta^\star$. SeqElim assumes strong stochastic transitivity, and Knockout also assumes the stochastic triangle inequality with constant $1$. These assumptions fail on the cyclic matrices, so we report how often the two methods return the exact winner.

\paragraph{WS-W and WS-S.}
We implement Algorithms~1--2 of \citet{chen2017dueling}. WS-S uses $\beta=1.1$ and plays $\lfloor\beta^\ell\rfloor$ self-play rounds after its $\ell$th exploration round. Its condition $1<\beta<p_{\min}/(1-p_{\min})$, where $p_{\min}=\min_{p_{ij}>1/2}p_{ij}$, holds on all our matrices, since $p_{\min}\ge0.525>1.1/2.1$.

\paragraph{WR-TINF and WR-EXP3-IX.}
We implement both algorithms of \citet{saad2024weak}. WR-TINF uses Tsallis exponent $2/3$, learning rate $2N^{-1/6}/\sqrt{t}$, and challenger weights proportional to $p_i^{2/3}$. After each duel, the sampled challenger, placed on the right, receives the importance-weighted loss $Y/q$, where $q$ is its conditional sampling probability and $Y=1$ means that the left arm wins. WR-EXP3-IX uses the source incumbent--challenger construction with rates $\eta_n=\sqrt{\log M/(Mn)}$ and $\gamma_n=\eta_n/2$ for $M=N-1$ challengers. Both algorithms target weak regret, and we also report their binary strong regret.

\paragraph{Versatile-DB.}
We use the two Tsallis-regularized learners of Algorithm~3 and Appendix~B of \citet{pmlr-v162-saha22a}, with exponent $1/2$ and learning rate $4/\sqrt{t}$. With $Y_t=1$ meaning that the left arm wins, the two loss estimates are $(1-Y_t)/p^{\rm left}_{i_t}$ and $Y_t/p^{\rm right}_{j_t}$.

\paragraph{WSW-PE and its continuation.}
WSW-PE follows Algorithm~3 of \citet{ma2025dual}, keeping cumulative pair counts and scores, with pair radius
\[
 \sqrt{\frac{\log 8+\log N-\log\delta+2\log n_{ij}}{2n_{ij}}}.
\]
Its $O(N\log N)$ guarantee assumes a total order, so we report total-order results next to the cyclic runs. Since WSW-PE stops after returning an arm, its online runs continue by repeatedly comparing that arm with a companion drawn from the algorithm seed, without knowledge of the winner (WSW-PE+ETC). The joint study also reports the WS-W process embedded in the source.

\FloatBarrier
\subsection{Results}
\label{app:experiment-results}

\paragraph{Main comparison.}
Table~\ref{tab:main-bai-audit} reports the full 100-seed BAI study behind Figure~\ref{fig:experiments}(a). Every method returned the correct winner in all 100 runs at each arm count. BKT uses fewer comparisons than TG-ITE-BAI throughout, and Knockout, which receives the gap-derived tolerance, is faster still. SeqElim has zero spread because its comparison budgets are deterministic.

Table~\ref{tab:main-regret-corrected} reports the endpoints of Figure~\ref{fig:experiments}(b,c) and the weak regret at $2\cdot10^5$ rounds. TG-ITE-Weak has a lower mean than WR-TINF, WR-EXP3-IX, and WS-W. WSW-PE+ETC has the lowest mean but a much larger spread. TG-ITE-KH has the lowest strong regret, and TG-ITE-Anytime ends above Versatile-DB, WR-TINF, and WS-S at $10^5$ rounds. WR-EXP3-IX always compares distinct arms, so its binary strong regret equals the horizon.

\begin{table}[htbp]
\centering
\caption{Main BAI comparison: comparisons in thousands, $\delta=.1$. All outputs are correct. SeqElim and Knockout receive $\epsilon=.304$.}
\label{tab:main-bai-audit}
\begin{tabular}{@{}rrrrrr@{}}
\toprule
$N$ & TG-ITE & BKT & SeqElim & Knockout & WSW-PE \\
\midrule
8 & $0.36\pm0.04$ & $0.33\pm0.04$ & $0.77\pm0.00$ & $0.25\pm0.03$ & $0.52\pm0.07$ \\
16 & $0.77\pm0.06$ & $0.69\pm0.06$ & $1.88\pm0.00$ & $0.52\pm0.04$ & $1.21\pm0.15$ \\
32 & $1.57\pm0.08$ & $1.36\pm0.08$ & $4.34\pm0.00$ & $1.06\pm0.06$ & $2.67\pm0.26$ \\
64 & $3.15\pm0.12$ & $2.71\pm0.10$ & $9.77\pm0.00$ & $2.11\pm0.08$ & $5.75\pm0.47$ \\
128 & $6.34\pm0.17$ & $5.36\pm0.16$ & $21.59\pm0.00$ & $4.20\pm0.13$ & $11.95\pm0.67$ \\
256 & $12.73\pm0.22$ & $10.65\pm0.18$ & $47.18\pm0.00$ & $8.36\pm0.17$ & $25.07\pm1.39$ \\
\bottomrule
\end{tabular}
\end{table}

\begin{table}[htbp]
\centering
\caption{Main online comparisons: binary regret in thousands. The weak and strong columns use the perturbed and the cyclic matrix of Figure~\ref{fig:experiments}, respectively.}
\label{tab:main-regret-corrected}
\begin{tabular}{@{}lrrr@{}}
\toprule
Algorithm & Weak, $5\!\cdot\!10^4$ & Weak, $2\!\cdot\!10^5$ & Strong, $10^5$ \\
\midrule
TG-ITE-Weak & $10.01\pm0.18$ & $10.01\pm0.18$ & -- \\
WS-W & $20.63\pm20.09$ & $30.00\pm37.52$ & -- \\
WR-TINF & $18.22\pm6.44$ & $20.55\pm8.00$ & $55.20\pm1.79$ \\
WR-EXP3-IX & $17.19\pm6.21$ & $17.66\pm8.77$ & $100.00\pm0.00$ \\
WSW-PE+ETC & $5.54\pm4.59$ & $5.54\pm4.59$ & -- \\
TG-ITE-KH & -- & -- & $22.02\pm0.27$ \\
TG-ITE-Anytime & -- & -- & $62.76\pm0.37$ \\
Versatile-DB & -- & -- & $25.89\pm1.49$ \\
WS-S & -- & -- & $56.35\pm20.78$ \\
\bottomrule
\end{tabular}
\end{table}

\paragraph{Global-gap sweeps.}
The sweeps fix $\Delta^\star=.30$ and lower the nonwinner gap $d$, so that $\Delta$ falls from $.30$ to $.025$. The known-gap variants receive $g=.30$. For BAI (Table~\ref{tab:gap-bai-corrected}), the cost of TG-ITE-BAI grows sharply as $\Delta$ shrinks, and so does that of BKT, which stays lower. The known-gap variant remains cheap because its capped tests need not resolve small nonwinner gaps. All three methods are correct in every run.

Online, at $\Delta=.05$ and $.025$, none of the 100 gap-oblivious warm starts or known-horizon identification passes finishes by $T=10^6$ (Tables~\ref{tab:gap-weak-corrected} and~\ref{tab:gap-strong-corrected}). These runs are unfinished, so weak regret is nearly or exactly $T$ and strong regret is exactly $T$. At $\Delta=.10$, identification finishes but dominates the cost. Mean weak regret is $229{,}326$ against $15{,}061$ for WR-TINF, and mean strong regret is $398{,}969$ against $78{,}359$ for WS-S. With $g=.30$, every warm start and identification pass finishes correctly at every gap. Mean weak regret then stays between $3{,}279$ and $8{,}257$, and mean strong regret between $7{,}335$ and $24{,}928$.

\begin{table}[htbp]
\centering
\caption{BAI gap sweep: comparisons in thousands, $N=256$, $\Delta^\star=.30$, $\delta=.1$. All outputs are correct. The $g$ variant receives $g=\Delta^\star=.30$.}
\label{tab:gap-bai-corrected}
\begin{tabular}{@{}rrrrr@{}}
\toprule
$d$ & $\Delta$ & TG-ITE-BAI & BKT & TG-ITE-$g$-BAI \\
\midrule
0.38 & 0.3 & $12.80\pm0.23$ & $10.74\pm0.20$ & $3.73\pm0.06$ \\
0.2 & 0.2 & $52.16\pm1.09$ & $44.93\pm1.12$ & $6.95\pm0.15$ \\
0.1 & 0.1 & $243.64\pm5.92$ & $216.08\pm5.48$ & $8.94\pm0.11$ \\
0.05 & 0.05 & $1{,}117.85\pm22.81$ & $1{,}008.63\pm23.27$ & $9.49\pm0.09$ \\
0.025 & 0.025 & $5{,}075.97\pm100.86$ & $4{,}648.08\pm100.42$ & $9.62\pm0.07$ \\
\bottomrule
\end{tabular}
\end{table}

\begin{table}[htbp]
\centering
\caption{Online weak-regret gap sweep: regret in thousands, $N=256$, $\Delta^\star=.30$, $T=10^6$. C/W/U refers to the warm starts. Oracle-$g$ receives $g=.30$ and is 100/0/0 in every row.}
\label{tab:gap-weak-corrected}
\begin{tabular}{@{}rrrrc@{}}
\toprule
$\Delta$ & Gap-oblivious & Oracle-$g$ & WR-TINF & Default C/W/U \\
\midrule
0.3 & $11.12\pm0.21$ & $3.28\pm0.07$ & $15.08\pm2.45$ & 100/0/0 \\
0.2 & $48.10\pm1.12$ & $6.08\pm0.15$ & $14.93\pm2.26$ & 100/0/0 \\
0.1 & $229.33\pm4.84$ & $7.71\pm0.09$ & $15.06\pm2.35$ & 100/0/0 \\
0.05 & $999.76\pm0.06$ & $8.15\pm0.07$ & $15.09\pm2.44$ & 0/0/100 \\
0.025 & $1{,}000.00\pm0.00$ & $8.26\pm0.06$ & $15.02\pm2.34$ & 0/0/100 \\
\bottomrule
\end{tabular}
\end{table}

\begin{table}[htbp]
\centering
\caption{Online strong-regret gap sweep: regret in thousands, $N=256$, $\Delta^\star=.30$, $T=10^6$. C/W/U refers to the identification passes. Oracle-$g$ receives $g=.30$ and is 100/0/0 in every row.}
\label{tab:gap-strong-corrected}
\begin{tabular}{@{}rrrrc@{}}
\toprule
$\Delta$ & Gap-oblivious & Oracle-$g$ & WS-S & Default C/W/U \\
\midrule
0.3 & $24.36\pm0.29$ & $7.33\pm0.09$ & $75.78\pm43.83$ & 100/0/0 \\
0.2 & $92.08\pm1.45$ & $13.78\pm0.26$ & $90.55\pm121.21$ & 100/0/0 \\
0.1 & $398.97\pm6.63$ & $20.93\pm0.29$ & $78.36\pm62.79$ & 100/0/0 \\
0.05 & $1{,}000.00\pm0.00$ & $24.09\pm0.20$ & $85.21\pm80.83$ & 0/0/100 \\
0.025 & $1{,}000.00\pm0.00$ & $24.93\pm0.15$ & $82.55\pm58.84$ & 0/0/100 \\
\bottomrule
\end{tabular}
\end{table}

\paragraph{Small winner gap.}
With $\Delta^\star=.05$ and nonwinner gap $.20$ (Table~\ref{tab:small-winner-corrected}), both TG-ITE warm starts finish correctly in all 30 runs, after $85{,}102\pm4{,}623$ comparisons for the gap-oblivious policy and $40{,}177\pm1{,}201$ for the known-gap policy. These counts exceed the weak regrets because comparisons with the winner incur no weak loss. Both policies incur far less weak regret than WR-TINF.

\begin{table}[htbp]
\centering
\caption{Small-winner-gap experiment: binary weak regret in thousands at $T=10^6$, $N=256$, $\Delta^\star=.05$, $d=.20$. The $g$ variant receives $g=.05$.}
\label{tab:small-winner-corrected}
\begin{tabular}{@{}lrr@{}}
\toprule
Algorithm & Mean $\pm$ sample SD & 95\% mean interval \\
\midrule
TG-ITE-Weak & $48.09\pm1.11$ & $[47.68,48.51]$ \\
TG-ITE-$g$-Weak & $32.19\pm0.73$ & $[31.92,32.46]$ \\
WR-TINF & $467.70\pm102.55$ & $[429.41,505.99]$ \\
\bottomrule
\end{tabular}
\end{table}

\paragraph{Cost of fresh samples.}
The ablation compares TG-ITE-Anytime with a heuristic that keeps win counts for every pair across phases and reuses them in later calls, using the same matrix, tree, start, and 100 paired seeds. This heuristic has no guarantee. At $T=10^6$, the fresh implementation uses $95{,}091$ new identification comparisons on average and the cache heuristic $32{,}339$. Their ratio is $2.940$, with paired 95\% bootstrap interval $[2.935,2.946]$ from 20,000 resamples. Both complete five identification passes in every run, and all their outputs are correct (Table~\ref{tab:fresh-phase-costs}). Phases 2--5 begin at the winner in every run and account for $87.59\%$ of the fresh comparisons. The first phase costs the same in both, so the saving comes from cheaper repeated confirmation.

\begin{table}[htbp]
\centering
\caption{New identification comparisons by phase, in thousands, over 100 paired seeds. The last column counts runs that enter the phase with the correct incumbent, which is the same for both.}
\label{tab:fresh-phase-costs}
\begin{tabular}{@{}crrc@{}}
\toprule
Phase & Fresh samples & Cache heuristic & Start at winner \\
\midrule
1 & $11.80\pm0.28$ & $11.80\pm0.28$ & 0/100 \\
2 & $13.22\pm0.23$ & $1.61\pm0.07$ & 100/100 \\
3 & $16.13\pm0.24$ & $2.70\pm0.10$ & 100/100 \\
4 & $21.62\pm0.27$ & $5.50\pm0.15$ & 100/100 \\
5 & $32.32\pm0.33$ & $10.73\pm0.19$ & 100/100 \\
\midrule
Total & $95.09\pm0.64$ & $32.34\pm0.37$ & -- \\
\bottomrule
\end{tabular}
\end{table}

\paragraph{Scaling in the horizon and the number of arms.}
Horizon scaling evaluates prefixes of one $10^6$-round trajectory per seed for TG-ITE-Weak and TG-ITE-Anytime, and it reruns TG-ITE-KH for each horizon (Table~\ref{tab:T-scaling-corrected}). Weak regret stays constant from $3\cdot10^4$ rounds on. TG-ITE-KH grows slowly, from $22{,}088$ at $10^5$ to $24{,}332$ at $10^6$. TG-ITE-Anytime rises to $95{,}214$ by $3\cdot10^5$ and then stays flat, because both later endpoints fall in the self-play block of its fifth phase. At $T=10^4$, no KH pass finishes, so its regret is $T$, while at the other horizons every pass finishes correctly.

Arm scaling compares the three online policies with WS-W and WS-S over 10 seeds per cell (Table~\ref{tab:N-scaling-corrected}). From $N=64$ to $1024$, mean regret divided by $N$ stays between $40.1$ and $45.2$ for TG-ITE-Weak, $94.5$ and $97.3$ for TG-ITE-KH, and $370.2$ and $379.8$ for TG-ITE-Anytime, so all three grow linearly in $N$. WS-W has lower weak regret only at $N=64$. TG-ITE-KH has lower strong regret than WS-S at every $N$, whereas TG-ITE-Anytime does not. In the additional 30-seed BAI sweep (Table~\ref{tab:BAI-N-scaling-corrected}), TG-ITE-BAI and BKT are correct in every run, and BKT is cheaper throughout. At $N=256$, lowering $\Delta$ from $.30$ to $.05$ raises the mean cost from $12{,}784$ to $1{,}112{,}906$ for TG-ITE-BAI and from $10{,}722$ to $1{,}005{,}190$ for BKT.

\begin{table}[htbp]
\centering
\caption{Horizon scaling at $N=256$, $(\Delta^\star,d)=(.30,.38)$: binary regret in thousands. Weak reports weak regret, and KH and Anytime report strong regret.}
\label{tab:T-scaling-corrected}
\begin{tabular}{@{}rrrr@{}}
\toprule
$T$ & Weak & KH & Anytime \\
\midrule
$10{,}000$ & $9.82\pm0.05$ & $10.00\pm0.00$ & $10.00\pm0.00$ \\
$30{,}000$ & $11.13\pm0.20$ & $20.92\pm0.28$ & $29.98\pm0.00$ \\
$100{,}000$ & $11.13\pm0.20$ & $22.09\pm0.30$ & $62.84\pm0.48$ \\
$300{,}000$ & $11.13\pm0.20$ & $23.16\pm0.29$ & $95.21\pm0.56$ \\
$1{,}000{,}000$ & $11.13\pm0.20$ & $24.33\pm0.31$ & $95.21\pm0.56$ \\
\bottomrule
\end{tabular}
\end{table}

\begin{table}[htbp]
\centering
\caption{Online arm-count scaling: binary regret in thousands, $T=10^6$, $(\Delta^\star,d)=(.30,.38)$. Weak and WS-W report weak regret, and KH, Anytime, and WS-S report strong regret.}
\label{tab:N-scaling-corrected}
\begin{tabular}{@{}rrrrrr@{}}
\toprule
$N$ & Weak & WS-W & KH & Anytime & WS-S \\
\midrule
64 & $2.56\pm0.09$ & $1.09\pm1.17$ & $6.23\pm0.12$ & $24.31\pm0.31$ & $13.68\pm1.20$ \\
128 & $5.38\pm0.14$ & $7.05\pm9.14$ & $12.37\pm0.21$ & $47.93\pm0.41$ & $32.40\pm9.20$ \\
256 & $11.15\pm0.25$ & $20.09\pm34.50$ & $24.35\pm0.34$ & $95.28\pm0.64$ & $70.84\pm34.18$ \\
512 & $22.80\pm0.35$ & $100.56\pm167.82$ & $48.57\pm0.37$ & $190.08\pm1.03$ & $200.75\pm165.25$ \\
1024 & $46.24\pm0.33$ & $138.47\pm118.15$ & $96.78\pm0.63$ & $379.04\pm1.17$ & $336.06\pm114.73$ \\
\bottomrule
\end{tabular}
\end{table}

\begin{table}[htbp]
\centering
\caption{Additional BAI arm-count scaling: comparisons in thousands, $\Delta^\star=.30$, $\delta=.1$. All outputs are correct. Column labels give $\Delta$, with nonwinner gap $.38$ or $.05$.}
\label{tab:BAI-N-scaling-corrected}
\begin{tabular}{@{}rrrrr@{}}
\toprule
$N$ & TG-ITE, $.30$ & BKT, $.30$ & TG-ITE, $.05$ & BKT, $.05$ \\
\midrule
8 & $0.40\pm0.06$ & $0.38\pm0.05$ & $17.12\pm2.86$ & $15.62\pm2.53$ \\
16 & $0.81\pm0.06$ & $0.75\pm0.06$ & $47.41\pm3.45$ & $43.93\pm3.63$ \\
32 & $1.61\pm0.08$ & $1.43\pm0.08$ & $111.55\pm7.99$ & $101.72\pm5.42$ \\
64 & $3.21\pm0.13$ & $2.78\pm0.12$ & $250.51\pm10.86$ & $228.80\pm7.73$ \\
128 & $6.40\pm0.15$ & $5.44\pm0.16$ & $535.00\pm17.07$ & $482.54\pm12.92$ \\
256 & $12.78\pm0.19$ & $10.72\pm0.19$ & $1{,}112.91\pm26.52$ & $1{,}005.19\pm22.20$ \\
\bottomrule
\end{tabular}
\end{table}

\paragraph{Tree and starting arm.}
With $N=256$ and $\delta=.1$, we ran TG-ITE-BAI on 50 random leaf orders with 10 shared seeds each. The mean cost per tree ranges from $12{,}744$ to $13{,}047$, a max/min ratio of $1.0238$, and all 500 outputs are correct. Varying the winner-entry level $h_\star$, the first ancestor level that contains the winner, over $0,\ldots,8$ with 30 seeds each gives mean costs between $12{,}700$ and $12{,}755$, again with all outputs correct. The cost is thus insensitive to the tree and the starting arm.

\paragraph{Joint output.}
At horizon $H=50{,}000$, the joint study compares TG-ITE-Dual, which runs a $\delta$-confidence TreeAscent and then the weak-regret policy, with TG-ITE-Weak, which keeps its confidence-$1/4$ warm start, and with WSW-PE+ETC. In Table~\ref{tab:joint-corrected}, every identification and warm start finishes with the correct winner. On the cyclic instance with $N=256$, Dual identifies after $12{,}732$ comparisons on average against $27{,}999$ for WSW-PE, but its weak regret is $12{,}035$ against $4{,}009$ for WSW-PE+ETC. On the total order, the weak regrets are $20{,}547$ and $560$. WSW-PE+ETC makes more comparisons yet incurs less weak regret, because comparisons with the winner are free. Over 30 seeds at $N=256$, the native WS-W process has weak regret $18{,}063\pm18{,}422$ on the cyclic instance and $526\pm472$ on the total order. In 800 further paired runs, with 100 seeds at each of four arm counts in both families, standalone TreeAscent and the identification stage of Dual return the same correct arm after the same number of comparisons.

\begin{table}[htbp]
\centering
\caption{Joint evaluation in thousands, $H=50{,}000$. $\tau$ is the $\delta=.1$ identification time for Dual and WSW-PE and the confidence-$1/4$ warm-start time for TG-ITE-Weak. The total-order assumption of WSW-PE holds only in the lower block.}
\label{tab:joint-corrected}
\begin{tabular}{@{}rlrrr@{}}
\toprule
$N$ & Statistic & TG-ITE-Dual & TG-ITE-Weak & WSW-PE+ETC \\
\midrule
\multicolumn{5}{l}{Cyclic Condorcet} \\
32 & $\tau\wedge H$ & $1.62\pm0.10$ & $1.48\pm0.09$ & $3.03\pm0.27$ \\
 & $R_H$ & $1.20\pm0.07$ & $1.10\pm0.07$ & $0.25\pm0.34$ \\
64 & $\tau\wedge H$ & $3.21\pm0.12$ & $2.96\pm0.12$ & $6.50\pm0.34$ \\
 & $R_H$ & $2.71\pm0.10$ & $2.49\pm0.10$ & $0.81\pm0.90$ \\
128 & $\tau\wedge H$ & $6.39\pm0.16$ & $5.90\pm0.16$ & $13.45\pm0.54$ \\
 & $R_H$ & $5.78\pm0.15$ & $5.33\pm0.13$ & $1.79\pm1.56$ \\
256 & $\tau\wedge H$ & $12.73\pm0.22$ & $11.74\pm0.21$ & $28.00\pm1.03$ \\
 & $R_H$ & $12.04\pm0.22$ & $11.08\pm0.20$ & $4.01\pm3.68$ \\
\multicolumn{5}{l}{Equal-gap total order} \\
32 & $\tau\wedge H$ & $2.47\pm0.15$ & $2.29\pm0.15$ & $2.92\pm0.19$ \\
 & $R_H$ & $2.06\pm0.14$ & $1.91\pm0.13$ & $0.07\pm0.07$ \\
64 & $\tau\wedge H$ & $5.13\pm0.21$ & $4.76\pm0.22$ & $6.27\pm0.26$ \\
 & $R_H$ & $4.62\pm0.20$ & $4.29\pm0.21$ & $0.15\pm0.12$ \\
128 & $\tau\wedge H$ & $10.49\pm0.31$ & $9.71\pm0.28$ & $13.09\pm0.37$ \\
 & $R_H$ & $9.88\pm0.29$ & $9.15\pm0.28$ & $0.28\pm0.21$ \\
256 & $\tau\wedge H$ & $21.24\pm0.45$ & $19.70\pm0.45$ & $27.33\pm0.50$ \\
 & $R_H$ & $20.55\pm0.44$ & $19.04\pm0.44$ & $0.56\pm0.36$ \\
\bottomrule
\end{tabular}
\end{table}

\paragraph{Anytime tails and failures.}
\label{app:experiment-tail}
We followed 100 TG-ITE-Anytime trajectories to $10^6$ rounds on the cyclic instance of Figure~\ref{fig:experiments}(c). The observed maximum is only $1.018$ times the mean at $10^5$ and $1.017$ times at $10^6$ (Table~\ref{tab:anytime-tail-corrected} and Figure~\ref{fig:anytime-tail-curves-corrected}). None of the 500 completed identification passes returns a wrong arm. Treating trajectories as independent trials, the one-sided exact 95\% upper bound on the probability that a run contains a wrong identification is $2.95\%$.

To observe failures, a stress test uses $N=32$, a cyclic matrix with $\Delta=\Delta^\star=g=.05$, loose confidence $\delta=.5$, and 5,000 seeds. Each run performs one TreeAscent$_g$ pass and then self-plays its output to $T=10^6$, so a wrong output incurs strong regret $T$. All 5,000 passes finish, and $27$ return a wrong arm, a rate of $0.54\%$ with Wilson interval $[0.371\%,0.785\%]$. Their regret $10^6$ is $39.50$ times the mean $25{,}316$ over correct runs. These 27 runs alone account for the gap between the all-run mean $30{,}579$ and the median $25{,}372$ (Table~\ref{tab:failure-stress-corrected}).

\begin{table}[htbp]
\centering
\caption{Tail statistics of TG-ITE-Anytime: strong regret in thousands over 100 trajectories, $N=256$, $(\Delta^\star,d)=(.30,.38)$. Both horizons use the same trajectories.}
\label{tab:anytime-tail-corrected}
\begin{tabular}{@{}rrrrr@{}}
\toprule
$T$ & Mean $\pm$ sample SD & Median & 95th percentile & Maximum \\
\midrule
$100{,}000$ & $62.81\pm0.45$ & $62.79$ & $63.51$ & $63.92$ \\
$1{,}000{,}000$ & $95.18\pm0.56$ & $95.18$ & $96.03$ & $96.77$ \\
\bottomrule
\end{tabular}
\end{table}

\begin{figure}[htbp]
\centering
\includegraphics[width=\linewidth]{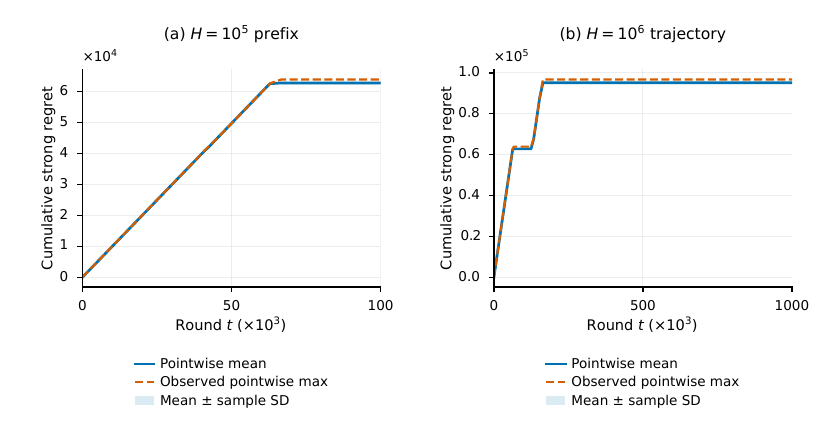}
\caption{Strong regret of the same 100 TG-ITE-Anytime trajectories up to $10^5$ and $10^6$. The solid line is the mean, the band is mean $\pm$ sample SD, and the dashed line is the pointwise maximum.}
\label{fig:anytime-tail-curves-corrected}
\end{figure}

\begin{table}[htbp]
\centering
\caption{Failure stress test: strong regret in thousands, grouped by the output of the pass.}
\label{tab:failure-stress-corrected}
\begin{tabular}{@{}lrrrrr@{}}
\toprule
Condition & Runs & Mean $\pm$ SD & Median & 95th pct. & Maximum \\
\midrule
All runs & 5{,}000 & $30.58\pm71.45$ & $25.37$ & $27.21$ & $1{,}000.00$ \\
Correct output & 4{,}973 & $25.32\pm1.14$ & $25.37$ & $27.15$ & $29.26$ \\
Wrong output & 27 & $1{,}000.00\pm0.00$ & $1{,}000.00$ & $1{,}000.00$ & $1{,}000.00$ \\
\bottomrule
\end{tabular}
\end{table}

\subsection{Computational resources}
\label{app:compute}
Our experiments are conducted on a server with an AMD Ryzen Threadripper 1950X CPU and 256GB RAM. With 32 parallel workers, all experiments complete within 5~minutes.

\FloatBarrier

\section{Potential Societal Impact}
\label{app:impact}
This work is primarily theoretical, and its experiments use only synthetic preference matrices. In ranking, recommendation, and human-preference learning, where people make the comparisons, needing fewer comparisons reduces the effort asked of them. The identified arm, however, is the Condorcet winner of the preferences actually collected, so it can favor some groups over others if the people who provide the comparisons do not represent those whom a system serves.

\end{document}